\documentclass{article}

\PassOptionsToPackage{numbers,compress,sort}{natbib}

\newif\ifisarxiv
\isarxivtrue
\ifisarxiv
\usepackage[preprint]{sty/nips2018/nips_2018}
\else
\usepackage[final]{sty/nips2018/nips_2018}
\fi

\usepackage[utf8]{inputenc} 
\usepackage[T1]{fontenc}    
\usepackage{hyperref}       
\usepackage{url}            
\usepackage{booktabs}       
\usepackage{amsmath}
\usepackage{color}
\usepackage{graphicx}
\usepackage{multirow}
\usepackage{amsfonts}       
\usepackage{nicefrac}       
\usepackage{microtype}      
\usepackage{wrapfig}
\usepackage{algorithm}
\usepackage{algorithmic}
\usepackage{xfrac}
\usepackage{caption}

\newcommand{\mnote}[1]{{\bf\large \Magenta{*}}\marginpar{\small \Magenta{#1}}}

\newcommand{\sets}[2] {{\hspace{-0.3mm}[\hspace{-0.3mm}#1\hspace{-0.3mm}]\hspace{-0.3mm}\choose \hspace{-0.3mm}#2\hspace{-0.3mm}}}

\def\r{\mathbf r}

\def\Q{\mathbf Q}

\ifx\BlackBox\undefined
\newcommand{\BlackBox}{\rule{1.5ex}{1.5ex}}  
\fi

\DeclareMathOperator*{\argmin}{\mathop{\mathrm{argmin}}}

\DeclareMathOperator*{\diag}{\mathop{\mathrm{diag}}}
\def\x{\mathbf x}
\def\y{\mathbf y}

\def\w{\mathbf w}
\def\v{\mathbf v}

\def\wbh{\widehat{\mathbf w}}

\def\e{\mathbf e}
\def\zero{\mathbf 0}
\def\one{\mathbf 1}
\def\u{\mathbf u}

\def\X{\mathbf X}

\def\A{\mathbf A}

\def\U{\mathbf U}

\def\M{\mathbf M}

\def\Z{\mathbf Z}
\def\Zbh{\widehat{\mathbf Z}}

\def\I{\mathbf I}

\def\A{\mathbf A}

\def\E{\mathbb E}
\def\R{\mathbb R} 
\def\Pr{\mathrm{Pr}} 
\def\tr{\mathrm{tr}}

\def\Var{\mathrm{Var}}

\newcommand{\defeq}{\stackrel{\textit{\tiny{def}}}{=}}

\newcommand{\cov}{\mathrm{cov}}
\let\origtop\top
\renewcommand\top{{\scriptscriptstyle{\origtop}}} 

\definecolor{silver}{cmyk}{0,0,0,0.3}
\definecolor{yellow}{cmyk}{0,0,0.9,0.0}
\definecolor{reddishyellow}{cmyk}{0,0.22,1.0,0.0}
\definecolor{black}{cmyk}{0,0,0.0,1.0}
\definecolor{darkYellow}{cmyk}{0.2,0.4,1.0,0}
\definecolor{darkSilver}{cmyk}{0,0,0,0.1}
\definecolor{grey}{cmyk}{0,0,0,0.5}
\definecolor{darkgreen}{cmyk}{0.6,0,0.8,0}

\newcommand{\Magenta}[1]{{\color{magenta}{#1}}}

\ifx\proof\undefined
\newenvironment{proof}{\par\noindent{\bf Proof\ }}{\hfill\BlackBox\\[2mm]}
\fi

\ifx\theorem\undefined
\newtheorem{theorem}{Theorem}
\fi

\ifx\example\undefined
\newtheorem{example}{Example}
\fi

\ifx\condition\undefined
\newtheorem{condition}{Condition}
\fi
\ifx\property\undefined
\newtheorem{property}{Property}
\fi

\ifx\lemma\undefined
\newtheorem{lemma}[theorem]{Lemma}
\fi

\ifx\proposition\undefined
\newtheorem{proposition}[theorem]{Proposition}
\fi

\ifx\remark\undefined
\newtheorem{remark}[theorem]{Remark}
\fi

\ifx\corollary\undefined
\newtheorem{corollary}[theorem]{Corollary}
\fi

\ifx\definition\undefined
\newtheorem{definition}{Definition}
\fi

\ifx\conjecture\undefined
\newtheorem{conjecture}[theorem]{Conjecture}
\fi

\ifx\axiom\undefined
\newtheorem{axiom}[theorem]{Axiom}
\fi

\ifx\claim\undefined
\newtheorem{claim}[theorem]{Claim}
\fi

\ifx\assumption\undefined
\newtheorem{assumption}[theorem]{Assumption}
\fi

\ifx\condition\undefined
\newtheorem{condition}[theorem]{Condition}
\fi

\title{Leveraged volume sampling for linear regression}

\author{
Micha{\l } Derezi\'{n}ski\\
Department of Computer Science\\
University of California at Santa Cruz\\
\texttt{mderezin@ucsc.edu}\\
\And
Manfred K. Warmuth\\
Department of Computer Science\\
University of California at Santa Cruz\\
\texttt{manfred@ucsc.edu}\\
\And
Daniel Hsu\\
Computer Science Department\\
Columbia University, New York\\
\texttt{djhsu@cs.columbia.edu}\\
}

\begin{document}

\maketitle

\begin{abstract}
 Suppose an $n \times d$ design matrix in a linear regression problem is given, 
but the response for each point is hidden unless explicitly requested. 
The goal is to sample only a small number $k \ll n$ of the responses, 
and then produce a weight vector whose sum of squares loss over \emph{all} points is at most $1+\epsilon$ times the minimum. 
When $k$ is very small (e.g., $k=d$), jointly sampling diverse subsets of
points is crucial. One such method called \emph{volume sampling} has
a unique and desirable property that the weight vector it produces is an unbiased
estimate of the optimum. It is therefore natural to ask if this method
offers the optimal unbiased estimate in terms of the number of
responses $k$ needed to achieve a $1+\epsilon$ loss approximation.

Surprisingly we show that volume sampling can have poor behavior when
we require a very accurate approximation -- indeed worse than some
i.i.d.~sampling techniques whose estimates are biased, such as
\emph{leverage score sampling}. 
We then develop a new rescaled variant of volume sampling that
produces an unbiased estimate which avoids
this bad behavior and has at least as good a tail bound as leverage
score sampling: sample size $k=O(d\log d + d/\epsilon)$ suffices to
guarantee total loss at most $1+\epsilon$ times the minimum
with high probability. Thus, we improve on the best previously known
sample size for an unbiased estimator, $k=O(d^2/\epsilon)$.

Our rescaling procedure leads to a new efficient algorithm
for volume sampling which is based
on a \emph{determinantal rejection sampling} technique with
potentially broader applications to determinantal point processes.
Other contributions include introducing the
combinatorics needed for rescaled volume sampling and developing tail
bounds for sums of dependent random matrices which arise in the
process.

\end{abstract}

\section{Introduction}
\label{s:intro}


Consider a linear regression problem where the input points
in $\R^d$ are provided, but the associated response for each point is
withheld unless explicitly requested. The goal is to
sample the responses for just a small subset of inputs,
and then produce a weight vector whose total square loss
on all $n$ points is at most $1+\epsilon$ times that of the
optimum.\footnote{The total loss of the algorithm being
at most $1+\epsilon$ times loss of the optimum can be rewritten
as the regret being at most $\epsilon$ times the optimum.}
This scenario is relevant in many applications where
data points are cheap to obtain but responses are expensive.
Surprisingly, with the aid of having all input points available,
such multiplicative loss bounds are achievable
without any range dependence on the points or responses common in
on-line learning \citep[see, e.g.,][]{onlineregr}.%

A natural and intuitive approach to this problem is
\emph{volume sampling}, since it
prefers ``diverse'' sets of points that will likely result in a
weight vector with low total loss, regardless of what the
corresponding responses turn out to be~\citep{unbiased-estimates}. Volume sampling is closely related
to optimal design criteria~\citep{optimal-design-book,dual-volume-sampling},
which are appropriate under statistical models of the responses;
here we study a worst-case setting where the algorithm must
use randomization to guard itself against worst-case responses.

Volume sampling and related determinantal point processes are employed in many
machine learning and statistical contexts, including linear
regression~\citep{dual-volume-sampling,unbiased-estimates,regularized-volume-sampling},
clustering and matrix
approximation~\citep{pca-volume-sampling,efficient-volume-sampling,avron-boutsidis13},
summarization and information retrieval~\citep{dpp,k-dpp,dpp-shopping}, and
fairness~\citep{celis2016fair,celis2018fair}. The availability of fast
algorithms for volume sampling~\citep{dual-volume-sampling,unbiased-estimates}
has made it an important technique in the algorithmic toolbox alongside
i.i.d.~leverage score sampling~\citep{drineas2006sampling} and spectral
sparsification~\citep{batson2012twice,lee2015constructing}.

It is therefore surprising that using volume sampling in the context of linear
regression, as suggested in previous
works~\citep{unbiased-estimates,dual-volume-sampling}, may lead to suboptimal
performance. We construct an example in which, even after sampling up to half
of the responses, the loss of the weight vector from volume sampling is a fixed
factor ${>}1$ larger than the minimum loss. Indeed,
this poor behavior arises because for any sample size ${>}d$, the marginal
probabilities from volume sampling are a mixture of uniform probabilities and
leverage score probabilities, and uniform sampling is well-known to be
suboptimal when the leverage scores are highly non-uniform.

\begin{wrapfigure}{r}{0.45\textwidth}
\vspace{-1.2cm}
\begin{center}
\includegraphics[width=.48\textwidth]{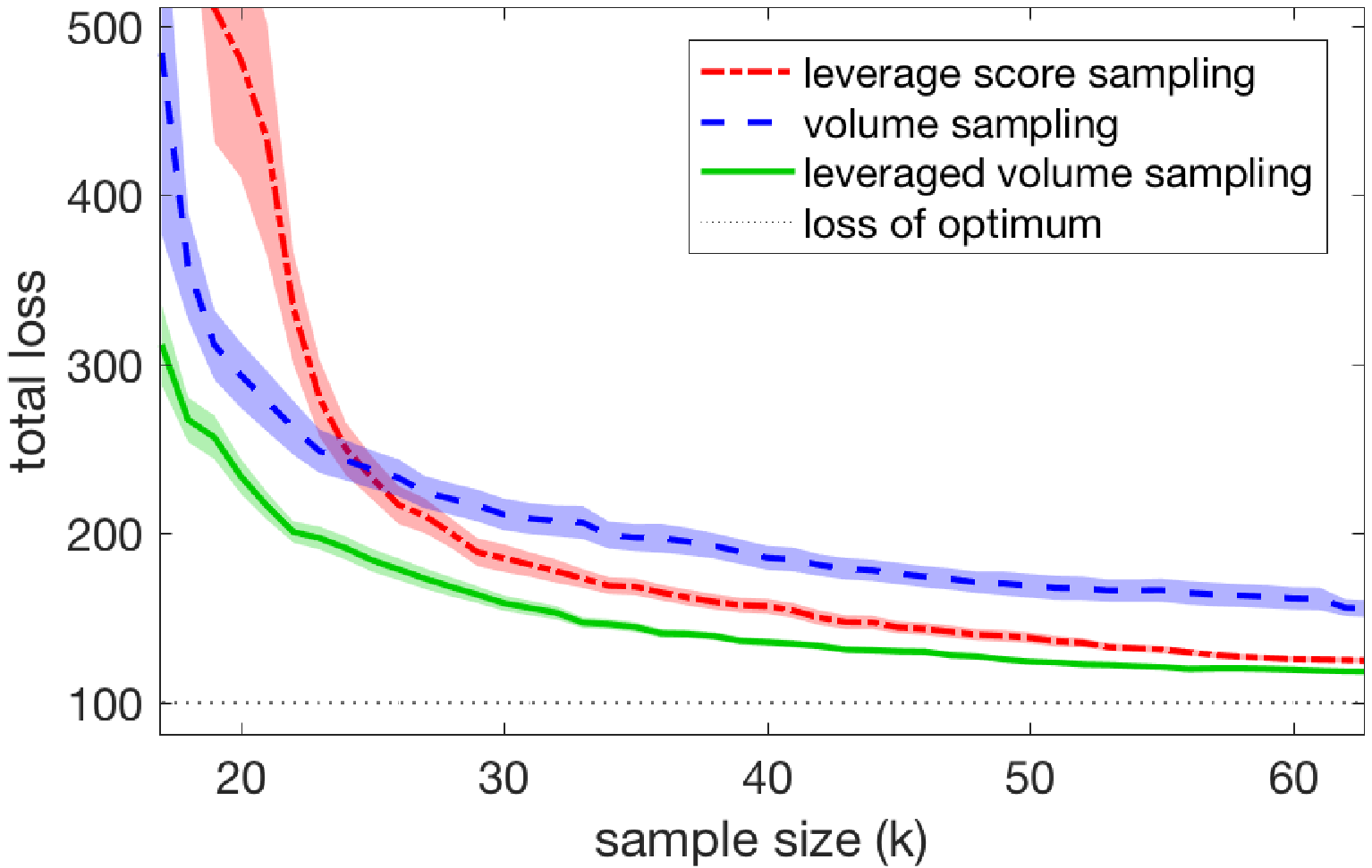}
   \captionof{figure}{Plots of the total loss for the sampling
methods (averaged over 100 runs) versus sample
size (shading is standard error) for a libsvm
dataset \textit{cpusmall} \cite{libsvm}.}
\label{f:lb}
\end{center}
\vspace{-.6cm}
\end{wrapfigure}

A possible recourse is to abandon volume sampling in favor of leverage score
sampling~\citep{drineas2006sampling,woodruff2014sketching}. However, all
i.i.d.~sampling methods, including leverage score sampling, suffer from a
coupon collector problem that prevents their effective use at small sample
sizes~\citep{regularized-volume-sampling}. Moreover, the resulting weight
vectors are biased (regarded as estimators for the least
squares solution using all responses), which is a nuisance when averaging
multiple solutions (e.g., as produced in distributed settings). In contrast,
volume sampling offers multiplicative loss bounds even with sample sizes as small as $d$
and it is the \textit{only} known non-trivial method that gives unbiased weight vectors~\citep{unbiased-estimates}.

We develop a new solution, called \emph{leveraged volume sampling}, that
retains the aforementioned benefits of volume sampling while avoiding its
flaws. Specifically, we propose a variant of volume sampling based on rescaling
the input points to ``correct'' the resulting marginals. On the
algorithmic side, this leads to
a new \textit{determinantal rejection sampling} procedure which offers significant
computational advantages over existing volume sampling algorithms,
while at the same time being strikingly simple to implement.
We prove that this new sampling scheme retains the benefits
of volume sampling (like unbiasedness) but avoids the bad behavior demonstrated
in our lower bound example. Along the way, we prove a new generalization of the
Cauchy-Binet formula, which is needed for the rejection sampling denominator.
Finally, we develop a new method for proving matrix tail bounds for leveraged
volume sampling. Our analysis shows that the unbiased
least-squares estimator constructed this way achieves a $1+\epsilon$
approximation factor from a sample of size $O(d \log d + d/\epsilon)$,
addressing an open question
posed by \cite{unbiased-estimates}.

\paragraph{Experiments.} 
Figure~\ref{f:lb} presents experimental evidence on a benchmark dataset
(\textit{cpusmall} from the libsvm collection \cite{libsvm}) that the
potential bad behavior of volume sampling proven in our lower bound
does occur in practice. Appendix
\ref{sec:experiments} shows more 
datasets and a detailed discussion of the experiments. In summary,
leveraged volume sampling avoids the bad behavior of standard volume
sampling, and performs considerably better than leverage score
sampling, especially for small sample sizes $k$.


\paragraph{Related work.}
Despite the ubiquity of volume sampling in many contexts already mentioned above, it has only recently been analyzed for linear regression.
Focusing on small sample sizes, \citep{unbiased-estimates} proved multiplicative bounds for the expected loss of size $k=d$ volume sampling.
Because the estimators produced by volume sampling are unbiased, averaging a number of such estimators produced an estimator based on a sample of size $k = O(d^2/\epsilon)$ with expected loss at most $1+\epsilon$ times the optimum.
It was shown in \cite{regularized-volume-sampling} that if the
responses are assumed to be linear functions of the input
points plus white noise, then size $k =
O(d/\epsilon)$ volume sampling suffices for obtaining the same
expected bounds. These noise assumptions on the response vector are
also central to the task of 
A-optimal design, where volume sampling is a key technique
\citep{optimal-design-book,symmetric-polynomials,tractable-experimental-design,proportional-volume-sampling}.
All of these previous results were concerned with bounds that hold in expectation; it is natural to ask if similar (or better) bounds can also be shown to hold with high probability, without noise assumptions.
Concentration bounds for volume sampling and other strong
Rayleigh measures were studied in
\cite{pemantle2014concentration}, but these results are
not sufficient to obtain the tail bounds for volume sampling.

Other techniques applicable to our linear regression
problem include leverage score
sampling~\citep{drineas2006sampling} and spectral
sparsification~\citep{batson2012twice,lee2015constructing}.
Leverage score sampling is an i.i.d. sampling procedure which achieves
tail bounds matching the ones we obtain here for leveraged volume
sampling, however it produces biased weight vectors 
and experimental results (see \cite{regularized-volume-sampling} and Appendix
\ref{sec:experiments}) show that it has weaker performance for small
sample sizes. 
A different and more elaborate 
sampling technique based on spectral
sparsification~\citep{batson2012twice,lee2015constructing} 
was recently shown to be effective for linear
regression~\citep{chen2017condition}, however this method
also does not
produce unbiased estimates, which is a primary concern of this paper
and desirable in many settings. Unbiasedness seems to
require delicate control of the sampling probabilities, which we achieve using determinantal rejection sampling.%

\paragraph{Outline and contributions.}
We set up our task of subsampling for linear regression in the next section
and present our lower bound for standard volume sampling.
A new variant of rescaled volume sampling is introduced
in Section \ref{s:resc}. We develop techniques for
proving matrix expectation formulas for this variant
which show that for any rescaling 
the weight vector produced for the subproblem is unbiased.

Next, we show that when rescaling with leverage scores, then
a new algorithm based on rejection sampling
is surprisingly efficient (Section \ref{s:alg}): 
Other than the preprocessing step of computing
leverage scores, the runtime does not depend on $n$
(a major improvement over existing volume sampling algorithms).
Then, in Section \ref{s:tail} we prove
multiplicative loss bounds for leveraged volume sampling
by establishing two important properties which are hard to prove for
joint sampling procedures.
We conclude in Section \ref{s:open} with an open problem
and with a discussion of how rescaling with approximate
leverage scores gives further time improvements for
constructing an unbiased estimator.

\section{Volume sampling for linear regression}
\label{s:versus}

In this section, we describe our linear regression setting, and review
the guarantees that standard volume sampling offers in this
context. Then, we present a surprising lower bound which shows that
under worst-case data, this method can exhibit undesirable behavior.

\subsection{Setting}
\label{s:setting}

Suppose the learner is given $n$ input vectors $\x_1,\dotsc,\x_n\in\R^d$, which
are arranged as the rows of an $n\times d$ input matrix $\X$.
Each input vector $\x_i$ has an associated response variable
$y_i\in \R$ from the response vector $\y\in\R^n$. The goal of the
learner is to find a weight vector $\w\in \R^d$ that minimizes the
square loss:
\begin{align*}
  \w^*\defeq \argmin_{\w\in\R^d} L(\w),
\;\;\text{where}\; L(\w)\defeq \sum_{i=1}^n
(\x_i^\top\w-y_i)^2=\|\X\w - \y\|^2.
\end{align*}
Given both matrix $\X$ and vector $\y$, the least squares solution can be
directly computed as $\w^* = \X^+\y$, where $\X^+$ is the
pseudo-inverse. Throughout the paper we assume w.l.o.g.~that
$\X$ has (full) rank $d$.%
\footnote{Otherwise just reduce $\X$ to a subset of independent columns. 
Also assume $\X$ has no rows of all zeros
(every weight vector has the same loss on such rows, so
they can be removed).}

In our setting, the learner is only given the input 
matrix $\X$, while response vector $\y$ remains hidden. 
The learner is allowed to select a
subset $S$ of row indices in $[n] = \{1,\dotsc,n\}$ for which the corresponding responses
$y_i$ are revealed. The learner constructs an estimate
$\wbh$ of $\w^*$
using matrix $\X$ and the partial vector of observed responses.
The learner is evaluated by the loss over all rows
of $\X$ (including the ones with unobserved responses), and the goal is to
obtain a multiplicative loss bound, i.e., that for some $\epsilon>0$,
\begin{align*}
L(\wbh)\leq (1+\epsilon)\,L(\w^*).
\end{align*}

\subsection{Standard volume sampling}

Given $\X\in\R^{n\times d}$ and a size $k\ge d$,
standard volume sampling jointly chooses a set $S$
of $k$ indices in $[n]$ with probability
\begin{align*}
\Pr(S) = \frac{\det(\X_S^\top\X_S)}{{n-d\choose k-d}\det(\X^\top\X)},
\end{align*}
where $\X_S$ is the submatrix of the rows from $\X$
indexed by the set $S$. The learner then obtains the responses $y_i$, for $i\in S$,
and uses the optimum solution $\w_S^*=(\X_S)^+\y_S$
for the subproblem $(\X_S,\y_S)$ as its weight vector. The sampling
procedure can be performed using \emph{reverse iterative sampling} (shown on the
right), which, if carefully implemented, takes $O(nd^2)$ time (see
\cite{unbiased-estimates,regularized-volume-sampling}).

\begin{wrapfigure}{R}{0.31\textwidth}
\renewcommand{\thealgorithm}{}
\ifisarxiv
\vspace{-3mm}
\else
\vspace{-6mm}
\fi
\hspace{-2mm}
\begin{minipage}{0.31\textwidth}
\floatname{algorithm}{}
\begin{algorithm}[H] 
{\fontsize{8}{8}\selectfont
  \caption{\bf \small Reverse iterative sampling}
  \begin{algorithmic}[0]
    \STATE VolumeSample$(\X,\,k)\!:$
    \STATE \quad$S \leftarrow [n]$
\vspace{1mm}
    \STATE \quad{\bf while} $|S|>k$
\vspace{-1.5mm}
    \STATE \quad\quad $\forall_{i\in S}\!:q_i\!\leftarrow\!
    \frac{\det(\X_{S\backslash i}^\top\!\X_{S\backslash i})}{\det(\X_S^\top\X_S)}$
    \STATE \quad\quad Sample $i\propto q_i$ out of $S$
\vspace{1mm}
    \STATE \quad\quad $S\leftarrow S \backslash \{i\}$
    \STATE \quad{\bf end} 
    \RETURN $S$
  \end{algorithmic}
}
\end{algorithm}
\end{minipage}
\end{wrapfigure}
The key property (unique to volume sampling) is
that the subsampled estimator $\w_S^*$ is unbiased, i.e.
\begin{align*}
\E[\w_S^*] = \w^*, \quad \text{where}\quad \w^* = \argmin_\w L(\w).
\end{align*}
As discussed in \cite{unbiased-estimates}, this property has important
practical implications in distributed settings: 
Mixtures of unbiased estimators remain unbiased (and can
conveniently be used to reduce variance).
Also if the rows of $\X$ are in general position, then
for volume sampling
\vspace{-1mm}
\begin{align}
\E\big[(\X_S^\top\X_S)^{-1}\big] =
  \frac{n-d+1}{k-d+1}\,(\X^\top\X)^{-1}.
\label{eq:square-inverse}
\end{align}
This is important because in A-optimal design bounding
$\tr((\X_S^\top\X_S)^{-1})$ is the main concern.
Given these direct connections of volume sampling to linear
regression, it is natural to ask whether this distribution
achieves a loss bound of $(1+\epsilon)$ times the optimum for
small sample sizes $k$.
\subsection{Lower bound for standard volume sampling}
\label{s:lower}

We show that standard volume sampling cannot
guarantee $1+\epsilon$ multiplicative loss bounds on some instances, unless over half of the rows are chosen to be
in the subsample.%
\begin{theorem}
\label{t:lower}
Let $(\X,\y)$ be an $n\times d$ least squares problem, such that
\begin{align*}
\X=\begin{pmatrix}
&\multirow{1}{*}{$\I_{d\times d}$}&\\
\hline
&\gamma\,\I_{d\times d}&\\
\hline
&\vdots&\\
\hline
&\gamma\,\I_{d\times d}&
\end{pmatrix} , \quad\y=
\begin{pmatrix}
\multirow{1}{*}{$\one_d$}\\
\hline
\zero_d\\
\hline
\vdots\\
\hline
\zero_d\end{pmatrix},\qquad \text{where}\quad \gamma>0.
\end{align*}
Let $\w_S^*=(\X_S)^+\y_S$ be
obtained from size $k$ volume sampling for $(\X,\y)$. Then,
\begin{align}
\lim_{\gamma\rightarrow 0}\frac{\E[L(\w_S^*)]}{L(\w^*)} \, \geq\, 
 1 + \frac{n-k}{n-d},\label{eq:L1}
\end{align}
and there is a $\gamma>0$ such that for any $k\leq \frac{n}{2}$,
\begin{align}
\Pr\bigg(L(\w_S^*) \geq \Big(1+\frac{1}{2}\Big)L(\w^*)\bigg) >
       \frac{1}{4}.\label{eq:L2}
\end{align}
\end{theorem}
\textbf{Proof }\ 
In Appendix \ref{a:lower} we show part \eqref{eq:L1}, and that for the chosen
$(\X,\y)$ we have $L(\w^*)\!=\!\sum_{i=1}^d \! (1\!-\!l_i)$ (see
\eqref{eq:lb-loss}), where $l_i=\x_i^\top(\X^\top\X)^{-1}\x_i$ is the
$i$-th leverage score of $\X$. Here, we show \eqref{eq:L2}.
The marginal probability of the $i$-th row under volume sampling
(as given by \cite{unbiased-estimates-journal}) is
\begin{align}
\Pr(i\in S) = \theta \ l_i + (1-\theta) \ 1
=  1 - \theta \ (1-l_i),\;
\text{  where }\theta= \frac{n-k}{n-d}.\label{eq:marginal}
\end{align}
Next, we bound the probability that all of the first $d$ input vectors
were selected by volume sampling:
\begin{align*}
  \Pr\big([d]\subseteq S\big)&
\overset{(*)}{\leq} \prod_{i=1}^d \Pr(i\in S)
=\prod_{i=1}^d\Big(1- \frac{n-k}{n-d}\,(1-l_i)\Big)
\leq \exp\Big(-\frac{n-k}{n-d}\!\overbrace{L(\w^*)}^{\sum_{i=1}^d  (1-l_i)}\!\!\Big),
\end{align*}
where $(*)$ follows from negative associativity of volume sampling (see \cite{dual-volume-sampling}).
  If for some $i\in[d]$ we have $i\not\in S$, then $L(\w_S^*)\geq
1$. So for $\gamma$ such that
$L(\w^*)=\frac{2}{3}$ and any $k\leq \frac{n}{2}$: 
\begin{align*}
\Pr\bigg(L(\w_S^*)\geq
\Big(1+\frac{1}{2}\Big)\overbrace{L(\w^*)}^{2/3}\bigg)
&\geq 1 - \exp\Big(\!-\frac{n-k}{n-d}\cdot\frac{2}{3}\Big)
\geq 1 - \exp\Big(\!-\frac{1}{2}\cdot\frac{2}{3}\Big)> \frac{1}{4}.
\hspace{1cm}\BlackBox
\end{align*}
Note that this lower bound only makes use of the negative
associativity of volume sampling and the form of the
marginals. However the tail bounds we prove in Section
\ref{s:tail} rely on more subtle properties of volume
sampling. We begin by creating a variant of volume sampling with
rescaled marginals.

\section{Rescaled volume sampling} \label{s:resc}
Given any size $k\geq d$, our goal is to
jointly sample 
$k$ row indices $\pi_1,\dots,\pi_k$ with replacement 
(instead of a {\em subset} $S$ of $[n]$ of size $k$, we get a
\emph{sequence} $\pi\in[n]^k$).  
The second difference to standard volume sampling is that
we rescale the $i$-th row (and response) by
$\frac{1}{\sqrt{q_i}}$, where
$q = (q_1,...,q_n)$ is any discrete distribution over the set of
row indices $[n]$, such that $\sum_{i=1}^nq_i=1$ and $q_i>0$ for all $i\in[n]$.
We now define $q$-rescaled size $k$ volume sampling as a
joint sampling distribution over $\pi\in[n]^k$, s.t.
\begin{align}
\text{$q$-rescaled size $k$ volume sampling:}
\qquad\Pr(\pi) \;\sim \;
\det\Big(\sum_{i=1}^k \frac{1}{q_{\pi_i}}\x_{\pi_i}\x_{\pi_i}^\top\Big)
\;\;
\prod_{i=1}^kq_{\pi_i}.\qquad\qquad
\label{eq:sampling}
\end{align}
Using the following rescaling matrix
$\Q_\pi\defeq\sum_{i=1}^{|\pi|}\frac1{q_{\pi_i}}\e_{\pi_i}\e_{\pi_i}^\top\ \in\R^{n\times
  n},$
we rewrite the determinant as $\det(\X^\top\Q_\pi\X)$.
As in standard volume sampling, the normalization factor in
rescaled volume sampling can be given in a closed form
through a novel extension of the Cauchy-Binet formula (proof in Appendix \ref{sec:cauchy-binet-proof}).

\begin{proposition}\label{p:cauchy-binet}
For any $\X\in\R^{n\times d}$, $k\geq d$ and $q_1,\dots,q_n>0$, such that $\sum_{i=1}^nq_i=1$,
  we have
\begin{align*}
\sum_{\pi\in[n]^k}\!\det(\X^\top\Q_\pi\X)
\,\prod_{i=1}^kq_{\pi_i} = k(k\!-\!1)...(k\!-\!d\!+\!1)\,\det(\X^\top\X).
\end{align*}
\end{proposition}
Given a matrix $\X\in\R^{n\times d}$, vector $\y\in\R^n$ and a
sequence $\pi\in[n]^k$, we are interested in a least-squares problem
$(\Q_\pi^{\sfrac{1}{2}}\X, \Q_\pi^{\sfrac{1}{2}}\y)$, which selects instances indexed by $\pi$,
and rescales each of them by the corresponding $1/\!\sqrt{q_i}$. This
leads to a natural subsampled least squares estimator
\begin{align*}
\w_\pi^*=\argmin_\w\sum_{i=1}^k\frac{1}{q_{\pi_i}}\big(\x_{\pi_i}^\top\w-y_{\pi_i}\big)^2
  = (\Q_\pi^{\sfrac{1}{2}}\X)^+\Q_\pi^{\sfrac{1}{2}}\y.
\end{align*}
The key property of standard volume sampling 
is that the subsampled least-squares estimator is unbiased.
Surprisingly this property is retained for any $q$-rescaled volume
sampling (proof in Section \ref{sec:unbiasedness-proof}).
As we shall see this will give us great leeway for choosing
$q$ to optimize our algorithms.
\begin{theorem}\label{t:unbiasedness}
Given a full rank $\X\in\R^{n\times d}$ and a response vector
$\y\in\R^n$, for any $q$ as above, if $\pi$ is sampled according to
\eqref{eq:sampling}, then
\begin{align*}
\E[\w_\pi^*] =\w^*,\quad\text{where}\quad
\w^*=\argmin_\w\|\X\w-\y\|^2.
\end{align*}
\end{theorem}
The matrix formula \eqref{eq:square-inverse}, discussed in Section
\ref{s:versus} for standard volume sampling, has a
natural extension to any rescaled volume sampling, turning here into an
inequality (proof in Appendix \ref{sec:square-inverse-proof}).
\begin{theorem}\label{t:square-inverse}
Given a full rank $\X\in\R^{n\times d}$ and any $q$ as above, if $\pi$ is sampled according to \eqref{eq:sampling}, then
\begin{align*}
\E\big[(\X^\top\Q_\pi\X)^{-1}\big]
\preceq \frac{1}{k\!-\!d\!+\!1}(\X^\top\X)^{-1}.
\end{align*}
\end{theorem}

\subsection{Proof of Theorem \ref{t:unbiasedness}}
\label{sec:unbiasedness-proof}
We show that the least-squares estimator
$\w_\pi^* = (\Q_\pi^{\sfrac{1}{2}}\X)^+\Q_\pi^{\sfrac{1}{2}}\y$
produced from any $q$-rescaled volume sampling is unbiased,
illustrating a proof technique which is also useful for showing
Theorem \ref{t:square-inverse},  as well as Propositions
\ref{p:cauchy-binet} and \ref{p:marginals}. The key idea is to apply
the pseudo-inverse expectation formula for standard volume sampling
(see e.g., \cite{unbiased-estimates}) first on the
subsampled estimator $\w_\pi^*$, and then again on the full estimator
$\w^*$. In the first step, this formula states:
\begin{align*}
\overbrace{(\Q_\pi^{\sfrac{1}{2}}\X)^+\Q_\pi^{\sfrac{1}{2}}\y}^{\w_\pi^*}=
\sum_{S\in\sets{k}{d}}
 \frac{\det(\X^\top\Q_{\pi_S}\X)}{\det(\X^\top\Q_\pi\X)} \overbrace{(\Q_{\pi_S}^{\sfrac{1}{2}}\X)^+\Q_{\pi_S}^{\sfrac{1}{2}}\y}^{\w_{\pi_S}^*},
\end{align*}
\vspace{-6mm}

where ${[k]\choose d} \defeq \{S\!\subseteq\! \{1,\dots,k\}:\ |S|\!=\!d\}$ and $\pi_S$ denotes a subsequence of $\pi$ indexed by the elements
of set $S$. Note that since $S$ is of size $d$, we can decompose
the determinant:
\begin{align*}
\det(\X^\top\Q_{\pi_S}\X) = \det(\X_{\pi_S})^2\,\prod_{i\in S}\frac{1}{q_{\pi_i}}.
\end{align*}
Whenever this determinant is non-zero, $\w_{\pi_S}^*$ is the
exact solution of a system of $d$ linear equations:
\begin{align*}
\frac{1}{\sqrt{q_{\pi_i}}}\x_{\pi_i}^\top\w=\frac{1}{\sqrt{q_{\pi_i}}}y_{\pi_i},\qquad\text{for}\quad
  i\in S.
\end{align*}
Thus, the rescaling of each equation by $\frac{1}{\sqrt{q_{\pi_i}}}$
cancels out, and we can simply write
$\w_{\pi_S}^*=(\X_{\pi_S})^+\y_{\pi_S}$. Note that this is not the case for
sets larger than $d$ whenever the optimum solution incurs positive loss.
We now proceed with summing over all $\pi\in[n]^k$. Following Proposition
\ref{p:cauchy-binet}, we define the normalization constant as
$Z=d!{k\choose d}\det(\X^\top\X)$, and obtain:
\begin{align*}
Z\,\E[\w_\pi^*]&=\!\!\!\sum_{\pi\in[n]^k}\!\!\bigg(\prod_{i=1}^kq_{\pi_i}\!\bigg)\det(\X^\top\Q_\pi\X)\,\w_\pi^*
=\!\!\!  \sum_{\pi\in[n]^k}\sum_{S\in\sets{k}{d}}
\!\bigg(\prod_{i\in[k]\backslash S}\!\!q_{\pi_i}\bigg) \det(\X_{\pi_S})^2  (\X_{\pi_S})^+\y_{\pi_S}\\
&\overset{(1)}{=}{k\choose  d}\sum_{\bar{\pi}\in[n]^d}\det(\X_{\bar{\pi}})^2  (\X_{\bar{\pi}})^+\y_{\bar{\pi}}
\sum_{\tilde{\pi}\in[n]^{k-d}}\prod_{i=1}^{k-d}q_{\tilde{\pi}_i}\\
&\overset{(2)}{=}{k\choose  d}d! \sum_{S\in\sets{n}{d}}
\det(\X_S)^2  (\X_S)^+\y_S\ 
\bigg(\sum_{i=1}^nq_i\bigg)^{k-d}\ 
\overset{(3)}{=}\overbrace{{k\choose d} d!\det(\X^\top\X)}^{Z}\,\w^*.
\end{align*}
Note that in $(1)$ we separate $\pi$ into two parts (subset $S$ and its
complement, $[k]\backslash S$) and sum over them separately. 
The binomial coefficient ${k\choose d}$ counts the number of ways that $S$ 
can be ``placed into'' the sequence $\pi$. In $(2)$ we observe that
whenever $\bar{\pi}$ has repetitions, determinant $\det(\X_{\bar{\pi}})$ is
zero, so we can switch to summing over sets. Finally, $(3)$ again uses the
standard size $d$ volume sampling unbiasedness formula, now for the least-squares
problem $(\X,\y)$, and the fact that $q_i$'s sum to 1.

\section{Leveraged volume sampling: a natural rescaling} \label{s:alg}

\begin{wrapfigure}{R}{0.44\textwidth}
\renewcommand{\thealgorithm}{}
\vspace{-8mm}
\begin{minipage}{0.44\textwidth}
\floatname{algorithm}{}
\begin{algorithm}[H] 
{\fontsize{8}{8}\selectfont
  \caption{\bf \small Determinantal rejection sampling}
  \begin{algorithmic}[1]
\vspace{-1mm}
    \STATE \textbf{Input:} $\X\!\in\!\R^{n\times d},
    q=(\frac{l_1}{d},\dots,\frac{l_n}{d}), k\geq d$
\vspace{-1mm}
    \STATE $s \leftarrow \max\{k,\,4d^2\}$
    \STATE \textbf{repeat}
\vspace{1mm}
    \STATE \quad Sample $\pi_1,\dots,\pi_{s}$ i.i.d. $\sim
    (q_1,\dots,q_n)$
    \STATE \quad Sample $\textit{Accept} \sim \text{Bernoulli}\Big(\frac{\det(\frac{1}{s}\X^\top\Q_\pi\X)}{\det(\X^\top\X)}\Big)$
\vspace{-1mm}
    \STATE \textbf{until} $\textit{Accept}=\text{true}$
    \STATE $S\leftarrow$
    VolumeSample$\big((\Q_{[1..n]}^{\sfrac{1}{2}}\X)_\pi,k\big)$
\vspace{-1mm}
    \RETURN $\pi_S$
 \end{algorithmic}
}
\end{algorithm}
\end{minipage}
\vspace{-0.8cm}
\end{wrapfigure}
Rescaled volume sampling can be viewed as selecting a
sequence $\pi$ of $k$ rank-1 covariates from the covariance matrix
$\X^\top\X = \sum_{i=1}^n\x_i\x_i^\top$. 
If $\pi_1,\dots,\pi_k$ are sampled i.i.d. from $q$, i.e.
$\Pr(\pi)=\prod_{i=1}^k q_{\pi_i}$, then matrix
$\frac{1}{k}\X^\top\Q_\pi\X$ is an unbiased estimator of the
covariance matrix because $\E[q_{\pi_i}^{-1} \x_{\pi_i}\x_{\pi_i}^\top]=\X^\top\X$.
In rescaled volume sampling \eqref{eq:sampling},
$\Pr(\pi)\sim$ 
$\big(\prod_{i=1}^k q_{\pi_i}\big)
\frac{\det(\X^\top\Q_\pi\X)}{\det(\X^\top\X)}$, 
and the latter volume
ratio introduces a bias to that estimator.
However, we show that this bias vanishes when $q$ is exactly
proportional to the leverage scores (proof in Appendix \ref{sec:marginals-proof}).

\begin{proposition}\label{p:marginals}
For any $q$ and $\X$ as before, if $\pi\in[n]^k$ is sampled according to
\eqref{eq:sampling}, then 
\begin{align*}
\E[\Q_\pi] = (k\!-\!d)\,\I + \diag\Big(\frac{l_1}{q_1},\dots,\frac{l_n}{q_n}\Big),
\quad\text{where}\quad l_i\defeq\x_i^\top(\X^\top\X)^{-1}\x_i.
\end{align*}
In particular, 
$\E[\frac{1}{k}\X^\top\Q_\pi\X]
=\X^\top\E[\frac{1}{k}\Q_\pi]\X
=\X^\top\X$ if and only if
$q_i=\frac{l_i}{d}>0$ for all $i\in[n]$.
\end{proposition}
This special rescaling, which we call \emph{leveraged volume sampling},
has other remarkable properties. Most importantly, it leads to a
simple and efficient algorithm we call {\em determinantal rejection sampling}: Repeatedly
sample $O(d^2)$ indices $\pi_1,\dots,\pi_s$ i.i.d. from
$q=(\frac{l_1}{d},\dots,\frac{l_n}{d})$,  and accept the sample with
probability proportional to its volume ratio. 
Having obtained a sample,
we can further reduce its size via reverse iterative sampling. 
We show next that this procedure not only returns a
$q$-rescaled volume sample, but also exploiting the fact
that $q$ is proportional to the leverage scores, 
it requires (surprisingly) only a constant number of
iterations of rejection sampling with high probability.

\begin{theorem}\label{t:algorithm}
Given the leverage score distribution
$q=(\frac{l_1}{d},\dots,\frac{l_n}{d})$ and the determinant
$\det(\X^\top\X)$ for matrix $\X\in\R^{n\times d}$, determinantal
rejection sampling returns sequence $\pi_S$
distributed according to leveraged volume sampling, and w.p. at
least $1\!-\!\delta$ finishes in time $O((d^2\!+k)d^2\ln(\frac{1}{\delta}))$.
\end{theorem}
\begin{proof}
We use a composition property of rescaled volume sampling
(proof in Appendix \ref{sec:composition-proof}):
\begin{lemma}\label{l:composition}
Consider the following sampling procedure, for $s>k$:
\begin{align*}
\pi\ &\overset{s}{\sim}\quad\X &&\text{($q$-rescaled size $s$ volume sampling)},\\
S\ &\overset{k}{\sim} \ 
\begin{pmatrix}
\frac{1}{\sqrt{q_{\pi_1}}}\x_{\pi_1}^\top\\
\dots\\
\frac{1}{\sqrt{q_{\pi_s}}}\x_{\pi_s}^\top\end{pmatrix}=
\big(\Q_{[1..n]}^{\sfrac12}\X\big)_\pi&& \text{(standard size $k$ volume sampling)}.
\end{align*}
Then $\pi_S$ is distributed according
to $q$-rescaled size $k$ volume sampling from $\X$.
\end{lemma}
First, we show that the rejection sampling probability in line 5 of
the algorithm is bounded by $1$:
\begin{align*}
\frac{\det(\frac{1}{s}\X^\top\Q_\pi\X)}{\det(\X^\top\X)}&
=\det\Big(\frac{1}{s}\X^\top\Q_\pi\X(\X^\top\X)^{-1}\Big)
\overset{(*)}{\leq}
  \bigg(\frac{1}{d}\tr\Big(\frac{1}{s}\X^\top\Q_\pi\X(\X^\top\X)^{-1}\Big)\bigg)^{\!d}\\
&=
  \Big(\frac{1}{ds}\tr\big(\Q_\pi\X(\X^\top\X)^{-1}\X^\top\big)\Big)^d
=\Big(\frac{1}{ds}\sum_{i=1}^s\frac{d}{l_i}\x_i^\top(\X^\top\X)^{-1}\x_i\Big)^d=1,
\end{align*}
where $(*)$ follows from the geometric-arithmetic mean
inequality for the eigenvalues of the underlying matrix. This shows
that sequence $\pi$ is drawn according to $q$-rescaled volume sampling
of size $s$. Now, Lemma \ref{l:composition} implies correctness of the algorithm.
Next,
we use  Proposition \ref{p:cauchy-binet} to compute the expected value
of acceptance probability from line 5 under the i.i.d. sampling of line 4:
\begin{align*}
\sum_{\pi\in[n]^s}\bigg(\prod_{i=1}^sq_{\pi_i}\bigg)\frac{\det(\frac{1}{s}\X^\top\Q_\pi\X)}{\det(\X^\top\X)}
&=\frac{s(s\!-\!1)\dots(s\!-\!d\!+\!1)}{s^d}\geq
  \Big(1-\frac{d}{s}\Big)^d\geq 1 - \frac{d^2}{s} \geq \frac{3}{4},
\end{align*}
where we also used Bernoulli's inequality and the fact that $s\geq
4d^2$ (see line 2). Since the expected value of the acceptance probability is at
least $\frac{3}{4}$, an easy application of Markov's inequality shows
that at each trial there is at least a 50\% chance of it being above
$\frac{1}{2}$. So, the probability of at least $r$ trials occurring is
less than $(1-\frac{1}{4})^r$. Note that the computational cost of one
trial is no more than the cost of SVD decomposition of matrix $\X^\top\Q_\pi\X$ (for
computing the determinant), which is $O(sd^2)$. The cost of
reverse iterative sampling (line 7) is also $O(sd^2)$ with high probability (as shown by
\cite{regularized-volume-sampling}). Thus, the overall runtime is
$O((d^2+k)d^2r)$, where $r\leq\ln(\frac{1}{\delta})/\ln(\frac{4}{3})$ w.p. at least $1-\delta$.
\end{proof}

\subsection{Tail bounds for leveraged volume sampling}
\label{s:tail}

An analysis of leverage score sampling, essentially following
\citep[Section 2]{woodruff2014sketching}
\citep[which in turn draws from][]{sarlos-sketching},
highlights two basic sufficient conditions on the 
(random) subsampling matrix $\Q_\pi$ that lead to
multiplicative tail bounds for $L(\w_\pi^*)$. 

It is convenient to shift to an orthogonalization of the linear regression task $(\X,\y)$
by replacing matrix $\X$ with a matrix
$\U=\X(\X^\top\X)^{-\sfrac12}\in\R^{n\times d}$. It is easy to
check that the columns of $\U$ have unit length and are
orthogonal, i.e., $\U^\top\U=\I$.
Now, $\v^*=\U^\top\y$ is the least-squares solution for the
orthogonal problem $(\U,\y)$ 
and prediction vector $\U\v^*=\U\U^\top\y$ for $(\U,\y)$ is the same as
the prediction vector $\X\w^*=\X(\X^\top\X)^{-1}\X^\top\y$
for the original problem $(\X,\y)$. 
The same property holds for the subsampled estimators, i.e.,
$\U\v_\pi^*=\X\w_\pi^*$, where $\v_\pi^*=
(\Q_\pi^{\sfrac{1}{2}}\U)^+\Q_\pi^{\sfrac{1}{2}} \,\y$. 
Volume sampling probabilities are also preserved under this transformation,
so w.l.o.g. we can work with the orthogonal problem.
Now $L(\v_\pi^*)$ can be rewritten as
\begin{align}
L(\v_\pi^*) =\|\U\v_\pi^*-\y\|^2\overset{(1)}{=} 
\|\U\v^*-\y\|^2
+ 
\|\U(\v_\pi^*-\v^*)\|^2 
\overset{(2)}{=}
L(\v^*) + \|\v_\pi^*-\v^*\|^2 ,
\label{e:pyth}
\end{align}
where $(1)$ follows via Pythagorean theorem from the fact that 
$\U(\v_\pi^*-\v^*)$ lies in the column span of $\U$ and
the residual vector $\r=\U\v^*-\y$ is orthogonal to all columns of $\U$, 
and $(2)$ follows from $\U^\top\U=\I$. 
By the definition of $\v_\pi^*$, we can write
$\|\v_\pi^*-\v^*\|^2$ as follows:
\begin{align}
\|\v_\pi^*-\v^*\| =
\|(\U^\top\Q_\pi\U)^{-1}\;\;\U^\top\Q_\pi(\y-\U\v^*)\|
\leq \|\underset{d\times d}{(\U^\top\Q_\pi\U)^{-1}}\|\,
      \|\underset{d\times 1}{\vphantom{(\U^\top\Q_\pi\U)^{-1}}\U^\top\Q_\pi\,\r}\|,
\label{e:prod}
\end{align}
where $\|\A\|$ denotes the matrix 2-norm 
(i.e., the largest singular value) of $\A$; when $\A$ is a
vector, then $\|\A\|$ is its Euclidean norm.
This breaks our task down to showing two key properties:
\begin{enumerate}
\item \textit{Matrix multiplication:}\quad Upper bounding the Euclidean norm $\|\U^\top\Q_\pi\,\r\|$,
\item \textit{Subspace embedding:}\quad Upper bounding the matrix 2-norm $\|(\U^\top\Q_\pi\U)^{-1}\|$.
\end{enumerate}

We start with a theorem that implies
strong guarantees for approximate matrix multiplication with leveraged
volume sampling. Unlike with i.i.d. sampling, this result requires
controlling the pairwise dependence
between indices selected under rescaled volume sampling. Its proof is
an interesting application of a classical Hadamard matrix
product inequality from
\cite{hadamard-product-inequality} (Proof in Appendix~\ref{sec:multiplication-proof}).
\begin{theorem}\label{t:multiplication}
Let $\U\in\R^{n\times d}$ be a matrix s.t. $\U^\top\U=\I$. 
If sequence $\pi\in[n]^k$ is selected using leveraged volume
sampling of size $k\geq \frac{2d}{\epsilon}$, then
for any $\r\in\R^n$,
\begin{align*}
\E\bigg[\Big\|\frac{1}{k}\U^\top\Q_\pi\r - \U^\top\r\Big\|^2\bigg] \leq
  \epsilon\, \|\r\|^2.
\end{align*}
\end{theorem}

Next, we turn to the subspace embedding property. The
following result is remarkable because
standard matrix tail bounds used to prove this property for leverage score
sampling are not applicable to volume sampling. In fact, obtaining
matrix Chernoff bounds for negatively associated joint distributions
like volume sampling is an active area of research, as discussed in
\cite{harvey2014pipage}. We address this challenge by
defining a coupling procedure for volume sampling and
uniform sampling without replacement, which leads to a curious
reduction argument described in Appendix \ref{sec:spectral-proof}.
\begin{theorem}
\label{t:spectral}
Let $\U\in\R^{n\times d}$ be a matrix s.t. $\U^\top\U=\I$. There is an
absolute constant $C$, s.t. if sequence $\pi\in[n]^k$ is selected
using leveraged volume sampling of size $k\geq
C\,d\ln(\frac{d}{\delta})$, then
\begin{align*}
\Pr\bigg(\lambda_{\min}\Big(\frac{1}{k}\U^\top\Q_\pi\U\Big)\leq
  \frac{1}{8}\bigg) \leq \delta.
\end{align*}
\end{theorem}
Theorems \ref{t:multiplication} and \ref{t:spectral} imply
that the unbiased estimator $\w_\pi^*$ produced from leveraged volume
sampling achieves multiplicative tail bounds with sample size
$k=O(d\log d + d/\epsilon)$.

\begin{corollary}
Let $\X\in\R^{n\times d}$ be a full rank matrix. There is an absolute constant
$C$, s.t.~if sequence $\pi\in[n]^k$ is selected using leveraged volume  
sampling of size $k\geq C\,\big(d\ln(\frac{d}{\delta}) +
\frac{d}{\epsilon\delta}\big)$, then for estimator
\begin{align*}
\w_\pi^* = \argmin_\w \|\Q_\pi^{\sfrac{1}{2}}(\X\w - \y)\|^2,
\end{align*}
we have $L(\w_\pi^*)\leq (1+\epsilon)\,L(\w^*)$ with probability at least $1-\delta$.
\end{corollary}
\textbf{Proof} \  Let 
$\U=\X(\X^\top\X)^{-\sfrac12}$. 
Combining Theorem \ref{t:multiplication} with
Markov's inequality, we have that for large enough $C$,
$\|\U^\top\Q_\pi\,\r\|^2\leq \epsilon\,\frac{k^2}{8^2}\|\r\|^2$ w.h.p., where
$\r=\y-\U\v^*$. Finally following (\ref{e:pyth}) and
(\ref{e:prod}) above, 
we have that w.h.p.
\begin{align*}
L(\w_\pi^*) &\leq L(\w^*) +
  \|(\U^\top\Q_\pi\U)^{-1}\|^2\,\|\U^\top\Q_\pi\,\r\|^2
\leq L(\w^*) + \frac{8^2}{k^2}\,\epsilon\frac{k^2}{8^2}\,\|\r\|^2
=(1+\epsilon)\,L(\w^*).\hspace{0.2cm} \BlackBox
\end{align*}

\section{Conclusion} \label{s:open}
We developed a new variant of volume sampling which produces the first 
known unbiased subsampled least-squares estimator with strong multiplicative
loss bounds. In the process, we proved a novel extension of the
Cauchy-Binet formula, as well as other fundamental combinatorial equalities.
Moreover, we proposed an efficient algorithm called determinantal
rejection sampling,
which is to our knowledge the first joint determinantal sampling
procedure that (after an initial $O(nd^2)$ preprocessing
step for computing leverage scores) produces its $k$ samples in time 
$\widetilde{O}(d^2\!+\!k)d^2)$, independent of the data size $n$. 
When $n$ is very large, the preprocessing time can be reduced to
$\widetilde{O}(nd + d^5)$ by rescaling 
with sufficiently accurate approximations of the leverage
scores. Surprisingly the estimator
stays unbiased and the loss bound still holds with only
slightly revised constants.
For the sake of clarity we presented the algorithm based
on rescaling with exact leverage scores in the main body of the paper.
However we outline the changes needed when using approximate
leverage scores in Appendix \ref{sec:fast-alg}.

In this paper we focused on tail bounds. However we conjecture that
expected bounds of the form $\E [L(\w_\pi^*)] \le (1+\epsilon)
L(\w^*)$ also hold for a variant of volume sampling of size $O(\frac{d}{\epsilon})$.


\bibliographystyle{plain}
\bibliography{pap}

\clearpage
\newpage

\appendix

\section{Proof of part \eqref{eq:L1} from Theorem~\ref{t:lower}}
\label{a:lower}
First, let us calculate $L(\w^*)$. Observe that
\vspace{-2mm}
\begin{align*}
(\X^\top\X)^{-1} &= \overbrace{\Big(1 +
  \frac{n-d}{d}\gamma^2\Big)^{-1}}^{c}\ \I,\\
\text{and}\quad \w^* &= c\,\X^\top\y = c\,\one_d.
\end{align*}
  The loss $L(\w)$ of any $\w \in \R^d$ can be decomposed as $L(\w) = \sum_{i=1}^d L_i(\w)$, where $L_i(\w)$ is the
total loss incurred on all input vectors $\e_i$ or $\gamma\e_i$:
\vspace{-4mm}
\begin{align*}
L_i(\w^*) = (1-c)^2 +
  \overbrace{\frac{n-d}{d}\gamma^2}^{\frac{1}{c}-1}\, c^2 
= 1-c,
\end{align*}
Note that $i$-th leverage score of $\X$ is equal $l_i=\x_i^\top(\X^\top\X)^{-1}\x_i=c$, so we obtain that
\begin{align}
L(\w^*)=d\,(1-c) = \sum_{i=1}^d(1-l_i).\label{eq:lb-loss}
\end{align}
Next, we compute $L(\w_S^*)$. Suppose that
$S\subseteq\{1..n\}$ is produced by size $k$ standard volume
sampling. Note that if for some $1\le i\le d$ we have $i\not\in S$, 
then $(\w_S^*)_i=0$ and therefore $L_i(\w_S^*)=1$. 
Moreover, denoting $b_i\defeq \one_{[i\in S]}$,
\begin{align*}
(\X_S^\top\X_S)^{-1} &\!\succeq\! (\X^\top\X)^{-1}\!=\!c\,\I,
\;\;\text{and}\;\; \X_S^\top\y_S \!=\! (b_1,\ldots,b_d)^\top\!,
\end{align*}
so if $i\in S$, then $(\w_S^*)_i\geq c$ and
\begin{align*}
L_i(\w_S^*) \geq \frac{n-d}{d} \,\gamma^2\,c^2 
= \Big(\frac{1}{c}-1\Big) c^2 = c\,L_i(\w^*).
\end{align*}
Putting the cases of $i\in S$ and $i\not\in S$ together, we get
\begin{align*}
L_i(\w_S^*) &\geq c\,L_i(\w^*) + (1-c\,L_i(\w^*))\,(1-b_i)\\
&\geq c\,L_i(\w^*) + c^2(1-b_i).
\end{align*}
Applying the marginal probability formula for volume sampling 
(see \eqref{eq:marginal}), we note that
\begin{align}
\nonumber
\E[1-b_i] &= 1-\Pr(i\in S) = \frac{n-k}{n-d}\,(1-c) = \frac{n-k}{n-d}\,L_i(\w^*).
\label{e:1mpi}
\end{align}
  Taking expectation over $L_i(\w_S^*)$ and summing the components over $i \in [d]$, we get 
\begin{align*}
\E[L(\w_S^*)] \geq L(\w^*)\Big(c+c^2\frac{n-k}{n-d}\Big).
\end{align*}
Note that as $\gamma\rightarrow 0$, we have $c\rightarrow 1$, thus showing \eqref{eq:L1}.


\section{Properties of rescaled volume sampling}
We give proofs of the properties of rescaled volume
sampling which hold for any rescaling distribution $q$. In this
section, we will use $Z=d!{k\choose d}\det(\X^\top\X)$ as the
normalization constant for rescaled volume sampling.

\subsection{Proof of Proposition \ref{p:cauchy-binet}}
\label{sec:cauchy-binet-proof}
First, we apply the Cauchy-Binet formula to
the determinant term specified by a fixed sequence $\pi\in[n]^k$:
\begin{align*}
\det(\X^\top\Q_\pi\X)= \sum_{S\in\sets{k}{d}}
  \det(\X^\top\Q_{\pi_S}\X)  = \sum_{S\in\sets{k}{d}}
\det(\X_{\pi_S})^2 \prod_{i\in S}\frac{1}{q_{\pi_i}}.
\end{align*}
Next, we compute the sum, using the above identity:
\begin{align*}
\sum_{\pi\in[n]^k}\!\!\det(\X^\top\Q_\pi\X)\prod_{i=1}^kq_{\pi_i}
&=\sum_{\pi\in[n]^k}\sum_{S\in\sets{k}{d}}\det(\X_{\pi_S})^2
\prod_{i\in[k]\backslash S}q_{\pi_i}\\
&={k\choose d}\!\!\sum_{\bar{\pi}\in[n]^d}\!\!\det(\X_{\bar{\pi}})^2\!\!
\sum_{\tilde{\pi}\in[n]^{k\!-\!d}}\prod_{i=1}^{k-d}q_{\tilde{\pi}_i}\\
&={k\choose d}\!\!\sum_{\bar{\pi}\in[n]^d}\!\!
\det(\X_{\bar{\pi}})^2\ \Big(\sum_{i=1}^nq_i\Big)^{k-d}\\
&={k\choose d}d!\sum_{S\in\sets{n}{d}}\det(\X_S)^2 
=k(k\!-\!1)...(k\!-\!d\!+\!1)\det(\X^\top\X),
\end{align*}
where the steps closely follow the corresponding derivation for Theorem
\ref{t:unbiasedness}, given in Section \ref{sec:unbiasedness-proof}.

\subsection{Proof of Theorem \ref{t:square-inverse}}
\label{sec:square-inverse-proof}
We will prove that for any vector $\v\in\R^d$,
\begin{align*}
\E\big[\v^\top(\X^\top\Q_\pi\X)^{-1}\v\big] \leq \frac{\v^\top(\X^\top\X)^{-1}\v}{k\!-\!d\!+\!1},
\end{align*}
which immediately implies the corresponding matrix inequality.
First, we use Sylvester's formula, which holds whenever
a matrix $\A\in\R^{d\times d}$ is full rank:
\begin{align*}
\det(\A+\v\v^\top) = \det(\A)\,
\big(1+\v^\top \A^{-1}\v\big).
\end{align*}
Note that whenever the matrix is not full rank, its determinant is
$0$ (in which case we avoid computing the matrix inverse), so we have for any $\pi\in[n]^k$:
\begin{align*}
\det(\X^\top\Q_\pi\X)\ \v^\top (\X^\top\Q_\pi\X)^{-1}\v
&\leq \det(\X^\top\Q_\pi\X+\v\v^\top) - \det(\X^\top\Q_\pi\X)\\
&\overset{(*)}{=}
\sum_{S\in\sets{k}{d\!-\!1}}\det(\X_{\pi_S}^\top\X_{\pi_S}+\v\v^\top)\prod_{i\in S}\frac{1}{q_{\pi_i}},
\end{align*}
where $(*)$ follows from applying the Cauchy-Binet formula to both of
the determinants, and cancelling out common terms. Next, we proceed
in a standard fashion, summing over all $\pi\in[n]^k$:
\begin{align*}
Z\ \E\big[
\v^\top (\X^\top\Q_\pi\X)^{-1}\v\big]
&=\sum_{\pi\in[n]^k}\!\! \v^\top
  (\X^\top\Q_\pi\X)^{-1}\v\det(\X^\top\Q_\pi\X)\prod_{i=1}^kq_{\pi_i}\\
&\leq
  \sum_{\pi\in[n]^k}\sum_{S\in\sets{k}{d\!-\!1}}\
  \!\!\!\det(\X_{\pi_S}^\top\X_{\pi_S}+\v\v^\top)
\prod_{i\in [k]\backslash S}q_{\pi_i}\\
&= {k\choose d\!-\!1} \sum_{\bar{\pi}\in[n]^{d-1}}
  \!\!\!\det(\X_{\bar{\pi}}^\top\X_{\bar{\pi}}+\v\v^\top)
\sum_{\tilde{\pi}\in[n]^{k-d+1}}\prod_{i=1}^{k-d+1}q_{\pi_i}\\
&= {k\choose d\!-\!1}(d\!-\!1)! \sum_{S\in\sets{n}{d\!-\!1}}
  \!\!\!\det(\X_S^\top\X_S+\v\v^\top)\\
&=\frac{d!{k\choose
  d}}{k\!-\!d\!+\!1}\big(\det(\X^\top\X+\v\v^\top)-\det(\X^\top\X)\big)
=Z\,\frac{\v^\top(\X^\top\X)^{-1}\v}{k\!-\!d\!+\!1}.
\end{align*}

\subsection{Proof of Proposition \ref{p:marginals}}
\label{sec:marginals-proof}
First, we compute the marginal probability of a fixed element of
sequence $\pi$ containing a particular index $i\in[n]$ under
$q$-rescaled volume sampling:
\begin{align*}
Z&\ \Pr(\pi_k\!=\!i) = \sum_{\pi\in[n]^{k-1}}\,\det(\X^\top\Q_{[\pi,i]}\X)
                  \ q_i\,\prod_{t=1}^{k-1}q_{\pi_t}\\
&=\underbrace{q_i\!\!\!\sum_{\pi\in[n]^{k\!-\!1}}\sum_{S\in\sets{k\!-\!1}{d}}\!\!\!\det(\X_{\pi_S})^2
\!\!\!  \prod_{t\in[k\!-\!1]\backslash S}\!\!\!q_{\pi_t}}_{T_1} + 
\!\underbrace{\sum_{\pi\in[n]^{k\!-\!1}}\sum_{S\in\sets{k\!-\!1}{d\!-\!1}}
\!\!\!\det(\X_{\pi_S}^\top\X_{\pi_S} + \x_i\x_i^\top) \!\!\!
\prod_{t\in[k\!-\!1]\backslash S}\!\!\!q_{\pi_t}}_{T_2},
\end{align*}
where the first term can be computed  by following the derivation in
Appendix \ref{sec:cauchy-binet-proof}, obtaining $T_1 =
q_i\frac{k-d}{k}\,Z$, and the second term is derived as in Appendix
\ref{sec:square-inverse-proof}, obtaining
$T_2=\frac{l_i}{k}\,Z$. Putting this together, we get
\begin{align*}
\Pr(\pi_k\!=\!i)=\frac{1}{k}\big((k\!-\!d)\,q_i + l_i\big).
\end{align*}
Note that by symmetry this applies to any element of the sequence. We
can now easily compute the desired expectation:
\begin{align*}
\E\big[(\Q_\pi)_{ii}\big] =\frac{1}{q_i} \sum_{t=1}^k\Pr(\pi_t\!=\!i)
  = (k\!-\!d) + \frac{l_i}{q_i}.
\end{align*}

\subsection{Proof of Lemma \ref{l:composition}}
\label{sec:composition-proof}
First step of the reverse iterative sampling procedure
described in Section \ref{s:versus} involves removing one row from the
given matrix with probability proportional to the square volume of that
submatrix:
\begin{align*}
\forall_{i\in S}\qquad \Pr(i\,|\,\pi_S)=
  \frac{\det(\X^\top\Q_{\pi_{S\backslash i}}\X)}
{(|S|-d)\det(\X^\top\Q_{\pi}\X)}.
\end{align*}
Suppose that $k=s-1$ and let
$\tilde{\pi}=\pi_S\in[n]^{s-1}$ denote the sequence obtained
after performing one step of the row-removal procedure. Then, 
\begin{align*}
\Pr({\tilde{\pi}})&=\sum_{i=1}^n\ s\
                    \overbrace{\Pr(i\,|\,[{\tilde{\pi}},i])}^{
\text{removing one row}}\quad
\overbrace{\Pr([{\tilde{\pi}},i])}^{\text{rescaled sampling}} \\
&= \sum_{i=1}^n \ s\ 
\frac{\det(\X^\top\Q_{\tilde{\pi}}\X)}{(s\!-\!d)\det(\X^\top\Q_{[{\tilde{\pi}},i]}\X)}\ 
\frac{\det(\X^\top\Q_{[{\tilde{\pi}},i]}\X)\,(\prod_{j=1}^{s-1}q_{\tilde{\pi}_j})\, q_i}{\frac{s!}{(s-d)!}\det(\X^\top\X)}\\
&=\frac{\det(\X^\top\Q_{{\tilde{\pi}}}\X)
(\prod_{j=1}^{s-1}q_{\tilde{\pi}_j})}{\frac{s-d}{s}\frac{s!}{(s-d)!}\det(\X^\top\X)}
\sum_{i=1}^nq_i = \frac{\det(\X^\top\Q_{{\tilde{\pi}}}\X)\, (\prod_{j=1}^{s-1}q_{\tilde{\pi}_j}) }{\frac{(s-1)!}{(s-1-d)!}\det(\X^\top\X)},
\end{align*}
where the factor $s$ next to the sum counts the number of ways to place index $i$
into the sequence $\tilde{\pi}$. Thus, by induction, for any $k<s$ the
algorithm correctly  samples from $q$-rescaled volume sampling.

\section{Proof of Theorem \ref{t:multiplication}}
\label{sec:multiplication-proof}
We rewrite the expected square norm as:
\begin{align*}
\E\bigg[\Big\|\frac{1}{k}\U^\top\Q_\pi\r - \U^\top\r\Big\|^2\bigg] 
&= \E\bigg[\Big\|\U^\top\!\Big(\frac{1}{k}\Q_\pi\!-\!\I\Big)\r\Big\|^2\bigg]
=\E\bigg[\r^\top\Big(\frac{1}{k}\Q_\pi\!-\!\I\Big)\U\U^\top\!
\Big(\frac{1}{k}\Q_\pi\!-\!\I\Big)\r\bigg]\\
&=\r^\top\ \E\bigg[ \Big(\frac{1}{k}\Q_\pi\!-\!\I\Big)\U\U^\top\!
\Big(\frac{1}{k}\Q_\pi\!-\!\I\Big)\bigg]\ \r\\
&\le \lambda_{\max}\Big(\underbrace{ \big(\E[(z_i\!-\!1)(z_j\!-\!1)]\,\u_i^\top\u_j\big)_{ij}}_{\M}\Big)\,\|\r\|^2,
\quad\text{where } z_i=\frac{1}{k}(\Q_\pi)_{ii}.
\end{align*}
It remains to bound $\lambda_{\max}(\M)$. By Proposition
\ref{p:marginals}, for leveraged volume sampling $\E[(\Q_\pi)_{ii}]=k$, so 
\begin{align*}
\E[(z_i\!-\!1)(z_j\!-\!1)] = \frac{1}{k^2}\Big(\E\big[(\Q_\pi)_{ii}(\Q_\pi)_{jj}\big] -
  \E\big[(\Q_\pi)_{ii}\big]\E\big[(\Q_\pi)_{jj}\big]\Big) =
  \frac{1}{k^2}\,\cov\big[(\Q_\pi)_{ii},\,(\Q_\pi)_{jj}\big].
\end{align*}
For rescaled volume sampling this is given in the
following lemma, proven in Appendix \ref{a:pairwise-formula}.
\begin{lemma}\label{l:pairwise-formula}
For any $\X$ and $q$, if sequence $\pi\in[n]^k$ is sampled from
$q$-rescaled volume sampling then
\begin{align*}
\cov\big[(\Q_\pi)_{ii},\,(\Q_\pi)_{jj}\big] =\one_{i=j} \frac{1}{q_i}\E\big[(\Q_\pi)_{ii}\big]
- (k\!-\!d) -
  \frac{(\x_i^\top(\X^\top\X)^{-1}\x_j)^2}{q_iq_j}.
\end{align*}
\end{lemma}
Since $\|\u_i\|^2=l_i=dq_i$ and
$\u_i^\top(\U^\top\U)^{-1}\u_j=\u_i^\top\u_j$, we can express matrix $\M$ as follows: 
\begin{align*}
\M = \diag\Big(\frac{d\ \E\big[(\Q_\pi)_{ii}\big]}{ \|\u_i\|^2k^2}\|\u_i\|^2\Big)_{i=1}^n
  -\frac{k\!-\!d}{k^2}\U\U^\top
- \frac{d^2}{k^2}\bigg(\frac{(\u_i^\top\u_j)^3}{\|\u_i\|^2\|\u_j\|^2}\bigg)_{ij}.
\end{align*}
The first term simplifies to $\frac{d}{k}\I$, and the second term is
negative semi-definite, so
\begin{align*}
\lambda_{\max}(\M) \leq \frac{d}{k} + \frac{d^2}{k^2}
\bigg\|\bigg(\frac{(\u_i^\top\u_j)^3}{\|\u_i\|^2\|\u_j\|^2}\bigg)_{ij}\bigg\|.
\end{align*}
 Finally, we decompose the last term into a Hadamard product 
of matrices, and apply a classical inequality by
\cite{hadamard-product-inequality}  (symbol ``$\circ$'' denotes
Hadamard matrix product):
\begin{align*}
\bigg\|\bigg(
  \frac{(\u_i^\top\u_j)^3}{\|\u_i\|^2\|\u_j\|^2}\bigg)_{\!ij}\bigg\|
\quad&=\quad
\bigg\|\bigg(
\frac{\u_i^\top\u_j}{\|\u_i\|\,\|\u_j\|}
\bigg)_{\!ij}
           \circ\bigg(
\frac{(\u_i^\top\u_j)^2}{\|\u_i\|\|\u_j\|}
\bigg)_{\!ij}\bigg\|\\
&\leq \quad\bigg\|\bigg(\frac{(\u_i^\top\u_j)^2}{\|\u_i\|\|\u_j\|}
\bigg)_{\!ij}\bigg\|\quad=\quad
\bigg\|\bigg(
\frac{\u_i^\top\u_j}{\|\u_i\|\,\|\u_j\|}
\bigg)_{\!ij}
           \circ\U\U^\top\bigg\|\\
&\leq \quad\|\U\U^\top\|\ =\  1.
\end{align*} 
Thus, we conclude that
$\E[\|\frac{1}{k}\U^\top\Q_\pi\r-\U^\top\r\|^2]\leq
(\frac{d}{k}+\frac{d^2}{k^2})\|\r\|^2$, completing the proof.

\subsection{Proof of Lemma~\ref{l:pairwise-formula}}
\label{a:pairwise-formula}

We compute marginal probability of two elements in the
sequence $\pi$ having particular values $i,j\in[n]$:
\begin{align*}
Z\,\Pr\big((\pi_{k-1}\!=\!i)\wedge(\pi_k\!=\!j)\big) &=
\sum_{\pi\in[n]^{k-2}}\sum_{S\in\sets{k}{d}}\det(\X_{[\pi,i,j]_S}^\top\X_{[\pi,i,j]_S})
\prod_{t\in[k]\backslash S}q_{[\pi,i,j]_t}.
\end{align*}
We partition the set ${[k]\choose d}$ of all subsets of size $d$ into
four groups, and summing separately over each of the groups, we have
\begin{align*}
Z\,\Pr\big((\pi_{k-1}\!=\!i)\wedge(\pi_k\!=\!j)\big) = T_{00} + T_{01}
  + T_{10} + T_{11},\qquad\text{where:}
\end{align*}
\begin{enumerate}
\item Let $G_{00} = \{S\!\in\! {[k]\choose d}:\ k\!-\!1\!\not\in \!S,\
  k\!\not\in\! S\}$, and following
  derivation in Appendix \ref{sec:cauchy-binet-proof}, 
\begin{align*}
T_{00} =
  q_i\,q_j\sum_{\pi\in[n]^{k-2}}\sum_{S\in G_{00}}
\det(\X_{\pi_S})^2
\prod_{t\in[k\!-\!2]\backslash S}q_{\pi_t} = q_i\,q_j\frac{(k\!-\!d\!-\!1)(k\!-\!d)}{(k\!-\!1)\,k}\,Z.
\end{align*}
\item Let $G_{10} = \{S\!\in\! {[k]\choose d}:\ k\!-\!1\!\in \!S,\
  k\!\not\in\! S\}$, and following
  derivation in Appendix \ref{sec:square-inverse-proof},
\begin{align*}
T_{10} =
q_j\sum_{\pi\in[n]^{k-1}}\sum_{S\in G_{10}}
\det(\X_{[\pi,i]_S})^2
\prod_{t\in[k\!-\!1]\backslash S}q_{[\pi,i]_t} = l_i\,q_j\frac{(k\!-\!d)}{(k\!-\!1)\,k}\,Z.
\end{align*}
\item $G_{01} = \{S\!\in\! {[k]\choose d}:\ k\!-\!1\!\not\in \!S,\ k\!\in\!
  S\}$, and by symmetry, $T_{01} =
  l_j\,q_i\frac{(k-d)}{(k-1)\,k}\,Z$.
\item Let $G_{11} = \{S\!\in\! {[k]\choose d}:\ k\!-\!1\!\in \!S,\
  k\!\in\!   S\}$, and the last term is
\begin{align*}
\hspace{-1cm}T_{11} &=
\sum_{\pi\in[n]^{k-1}}\sum_{S\in G_{11}}
\det(\X_{[\pi,i,j]_S})^2
\prod_{t\in[k]\backslash S}q_{[\pi,i,j]_t} \\
&={k\!-\!2\choose
  d\!-\!2}\sum_{\pi\in[n]^{d-2}}\det(\X_{[\pi,i,j]})^2\\
&={k\!-\!2\choose d\!-\!2}(d\!-\!2)!\,
\big(\det(\X^\top\X) -
 \det(\X_{-i}^\top\X_{-i}) - \det(\X_{-j}^\top\X_{-j})+
  \det(\X_{-i,j}^\top\X_{-i,j}) \big)\\
&\overset{(*)}{=}\frac{d!{k\choose
  d}}{k(k\!-\!1)}\det(\X^\top\X)\Big(1
  -\!\!\underbrace{(1\!-\!l_i)}_{\frac{\det(\X_{-i}^\top\X_{-i})}{\det(\X^\top\X)}}
\!\!  -\!\!
  \underbrace{(1\!-\!l_j)}_{\frac{\det(\X_{-j}^\top\X_{-j})}{\det(\X^\top\X)}}
 \!\! +
\underbrace{(1\!-\!l_i)(1\!-\!l_j) -l_{ij}^2}_{\frac{\det(\X_{-i,j}^\top\X_{-i,j})}{\det(\X^\top\X)}}\Big)\\[-2mm]
&=\frac{Z}{k(k\!-\!1)}\big(\ell_i\ell_j - \ell_{ij}^2\big),
\end{align*}
\end{enumerate} 
where $l_{ij}=\x_i^\top(\X^\top\X)^{-1}\x_j$, and $(*)$ follows from
repeated application of Sylvester's determinant formula (as in
Appendix \ref{sec:square-inverse-proof}). Putting it all together, we
can now compute the expectation for $i\neq j$:
\begin{align*}
\E\big[(\Q_\pi)_{ii}\,(\Q_\pi)_{jj}\big] &= 
\frac{1}{q_i\,q_j}\sum_{t_1=1}^k\sum_{t_2=1}^k
\Pr\big((\pi_{k-1}\!=\!i)\wedge(\pi_k\!=\!j)\big)\\
&=\frac{k(k\!-\!1)}{q_i\,q_j}  \overbrace{\Pr\big((\pi_{k-1}\!=\!i)\wedge(\pi_k\!=\!j)\big)}
^{\frac{1}{Z}(T_{00}+T_{10}+T_{01}+T_{11})}\\
&= (k\!-\!d\!-\!1)(k\!-\!d) + (k\!-\!d)\frac{l_i}{q_i} +
  (k\!-\!d)\frac{l_j}{q_j} + \frac{l_il_j}{q_i\,q_j} - \frac{l_{ij}^2}{q_i\,q_j}\\
&=\Big((k\!-\!d)q_i+\frac{l_i}{q_i}\Big)\Big((k\!-\!d)q_j +
  \frac{l_j}{q_j}\Big) - (k\!-\!d) - \frac{l_{ij}^2}{q_i\,q_j}\\
&=\E\big[(\Q_\pi)_{ii}\big]\,\E\big[(\Q_\pi)_{jj}\big] -(k\!-\!d) - \frac{l_{ij}^2}{q_iq_j}.
\end{align*}
Finally, if $i=j$, then
\begin{align*}
\E[(\Q_\pi)_{ii}\,(\Q_\pi)_{ii}] &= \frac{1}{q_i^2}\sum_{t_1=1}^k\sum_{t_2=1}^k\Pr(\pi_{t_1}\!=\!i\,
  \wedge\, \pi_{t_2}\!=\!i) \\
&= \frac{k(k\!-\!1)}{q_i^2} \,\Pr(\pi_{k-1}\!=\!i\, \wedge\,
  \pi_{k}\!=\!i) + \frac{k}{q_i^2}\, \Pr(\pi_k\!=\!i)\\
&=\big(\E\big[(\Q_\pi)_{ii}\big]\big)^2 -(k\!-\!d) -
  \frac{l_i^2}{q_i^2} +\frac{1}{q_i}\E\big[(\Q_\pi)_{ii}\big].
\end{align*}

\section{Proof of Theorem \ref{t:spectral}}
\label{sec:spectral-proof}
We break the sampling procedure down into two stages. First, we do leveraged volume
sampling of a sequence $\pi\in[n]^{m}$ of size $m\geq C_0 d^2/\delta$,
then we do standard volume
sampling size $k$ from matrix $(\Q_{[1..n]}^{\sfrac{1}{2}}\U)_\pi$. Since
rescaled volume sampling is closed under this
subsampling (Lemma \ref{l:composition}),
this procedure is equivalent to size $k$ leveraged volume sampling
from $\U$. To show that the first stage satisfies the subspace
embedding condition, we simply use the bound from Theorem
\ref{t:multiplication} (see details in Appendix \ref{a:overestimate}):
\begin{lemma}\label{l:overestimate}
There is an absolute constant $C_0$, s.t.~if sequence $\pi\in[n]^m$ is
generated via leveraged volume sampling of size $m$ at least
$C_0\,d^2/\delta$ from $\U$, then
\begin{align*}
\Pr\bigg(\lambda_{\min}\Big(\frac{1}{m}\U^\top\Q_\pi\U\Big) \leq
  \frac{1}{2}\bigg)\leq \delta.
\end{align*}
\end{lemma}
The size of $m$ is much larger than what we claim is sufficient.
However, we use it to achieve a tighter bound in the second stage.
To obtain substantially smaller sample sizes for subspace embedding than what Theorem~\ref{t:multiplication} can deliver, it is standard to use tail bounds for the
sums of independent matrices. However, applying these results to joint
sampling is a challenging task. Interestingly,
\cite{dual-volume-sampling} showed that volume sampling is a strongly
Raleigh measure, implying that the sampled vectors are negatively
correlated. This guarantee is sufficient to show tail bounds for
real-valued random variables \citep[see,
e.g.,][]{pemantle2014concentration},
however it has proven challenging
in the matrix case, as discussed by \cite{harvey2014pipage}. One
notable exception is uniform sampling without replacement, which is a
negatively correlated joint distribution. A reduction argument originally proposed
by \cite{hoeffding-with-replacement}, but presented in this context by
\cite{uniform-matrix-sampling}, shows that uniform sampling without
replacement offers the same tail bounds as i.i.d.~uniform sampling.

\begin{lemma}\label{l:without-replacement}
Assume that $\lambda_{\min}\big(\frac{1}{m}\U^\top\Q_\pi\U\big)\geq
\frac{1}{2}$. Suppose that set $T$ is a set of fixed size 
sampled uniformly without replacement from $[m]$. There is a constant
$C_1$ s.t.~if $|T|\ge  C_1\,d\ln(d/\delta)$, then 
\begin{align*}
\Pr\Big(\lambda_{\min}\Big(\frac{1}{|T|}\U^\top\Q_{\pi_T}\U\Big) \leq
  \frac{1}{4}\Big) \leq \delta.
\end{align*}
\end{lemma}
\vspace{-1mm}
The proof of Lemma \ref{l:without-replacement} (given in appendix
\ref{a:without-replacement}) is a straight-forward application of the
argument given by \cite{uniform-matrix-sampling}. We now propose a
different reduction argument showing that a subspace
embedding guarantee for uniform sampling without replacement leads to
a similar guarantee for volume sampling. We achieve this by exploiting a
volume sampling algorithm proposed recently by
\cite{regularized-volume-sampling}, shown in Algorithm
\ref{alg:volume}, which is a modification of the reverse iterative
sampling procedure introduced in \cite{unbiased-estimates}. This
procedure relies on iteratively removing elements from the set $S$
until we are left with $k$ elements. Specifically, at each step, we
sample an index $i$ from a conditional distribution, $i\sim \Pr(i\,|\,S)=(1-\u_i^\top(\U^\top\Q_{\pi_S}\U)^{-1}\u_i)/(|S|-d)$. Crucially for us, each step
proceeds via rejection sampling with the proposal distribution being
uniform. We can easily modify the algorithm, so that the
samples from the proposal distribution are used to construct a uniformly
sampled set $T$, as shown in Algorithm \ref{alg:coupled}. Note that
sets $S$ returned by both algorithms are identically distributed, and
furthermore, $T$ is a subset of $S$, because every index taken out of
$S$ is also taken out of $T$.

\noindent\begin{minipage}{\textwidth}
   \centering
   \begin{minipage}{.45\textwidth}
\small
     \centering
     \captionof{algorithm}{Volume sampling}
     \label{alg:volume}
\vspace{-.3cm}
     \begin{algorithmic}[1]
       \STATE $S \leftarrow [m]$
       \STATE {\bf while} $|S|>k$
       \STATE \quad \textbf{repeat}
       \STATE \quad\quad Sample $i$ unif. out of $S$
       \STATE \quad\quad $q \leftarrow 1-\u_i^\top (\U^\top\Q_{\pi_S}\U)^{-1}\u_i$
       \STATE \quad\quad Sample $\textit{Accept} \sim \text{Bernoulli}(q)$
       \STATE \quad \textbf{until} $\textit{Accept}=\text{true}$
       \STATE \quad $S\leftarrow S \backslash \{i\}$
       \STATE {\bf end} 
       \RETURN $S$
     \end{algorithmic}
   \end{minipage}
   \begin{minipage}{.45\textwidth}
\small
     \centering \vspace{5 mm}

     \captionof{algorithm}{Coupled sampling}
     \label{alg:coupled}
\vspace{-.3cm}
     \begin{algorithmic}[1]
       \STATE $S,T \leftarrow [m]$
       \STATE {\bf while} $|S|>k$
       \STATE \quad Sample $i$ unif. out of $[m]$
       \STATE \quad $T \leftarrow T - \{i\}$
       \STATE \quad {\bf if} $i\in S$
       \STATE \quad\quad $q \leftarrow 1-\u_i^\top
       (\U^\top\Q_{\pi_S}\U)^{-1}\u_i$
       \STATE \quad\quad Sample $\textit{Accept} \sim \text{Bernoulli}(q)$
       \STATE \quad\quad {\bf if} $\textit{Accept}=\text{true}$,\quad $S\leftarrow S\backslash\{i\}$ {\bf end}
       \STATE \quad{\bf end}
       \STATE {\bf end}
       \RETURN $S,T$
     \end{algorithmic}
   \end{minipage}
\vspace{5mm}
\end{minipage}

By Lemma \ref{l:without-replacement}, if size of $T$ is at least $ C_1\,d\log(d/\delta)$, then this set
offers a subspace embedding guarantee. Next, we will show that in
fact set $T$ is not much smaller than $S$, implying that the same
guarantee holds for $S$. Specifically, we will show that $|S \setminus T|=
O(d\log(d/\delta))$. Note that it suffices to bound the number of times
that a uniform sample is rejected by sampling $A=0$ in line 7 of
Algorithm \ref{alg:coupled}. Denote this number by $R$. Note that
$R=\sum_{t=k+1}^m R_t$, where $m=|Q|$ and $R_t$ is the number of times
that $A=0$ was
sampled while the size of set $S$ was $t$. Variables $R_t$
are independent, and each is
distributed according to the geometric distribution (number of
failures until success), with the success probability
\begin{align*}
r_t = \frac{1}{t}\sum_{i\in S}
\big(1-\u_i^\top(\U^\top\Q_{\pi_S}\U)^{-1}\u_i\big)
= \frac{1}{t}\Big(t-\tr\big((\U^\top\Q_{\pi_S}\U)^{-1}\U^\top\Q_{\pi_S}\U \big) \Big)
=\frac{t-d}{t}.
\end{align*}
\vspace{-.2cm}
Now, as long as $\frac{m-d}{k-d}\leq C_0\,d^2/\delta$,  we can bound the
expected value of $R$ as follows: 
\begin{align*}
\E[R] &=\!\sum_{t=k+1}^m\!\E[R_t]=\!\!\sum_{t=k+1}^m\!\!\Big(\frac{t}{t-d}-1\Big)
  =d\!\!\sum_{t=k-d+1}^{m-d}\frac{1}{t} \leq  d\,\ln\!\Big(\frac{m-d}{k-d}\Big)\leq  C_2\,d\ln(d/\delta).
\end{align*}
In this step, we made use of the first stage sampling, guaranteeing that
the term under the logarithm is bounded. Next, we show that the
upper tail of $R$ decays very rapidly given a sufficiently large gap
between $m$ and $k$ (proof in Appendix \ref{a:geometric-tail}):
\begin{lemma}\label{l:geometric-tail}
Let $R_t\sim \operatorname{Geom}(\frac{t-d}{t})$ be a sequence of independent geometrically
distributed random variables (number of failures until success). Then,
for any $d<k<m$ and $a>1$,
\begin{align*}
\Pr\big(R \geq a\ \E[R]\big) \leq
  \text{e}^{\frac{a}{2}}\,\Big(\frac{k-d}{m-d}\Big)^{\frac{a}{2}-1}\quad
\text{for}\quad R=\sum_{t=k+1}^m R_t.
\end{align*}
\end{lemma}
Let $a=4$ in Lemma \ref{l:geometric-tail}. Setting $C = C_1+2a\,C_2$,
for any $k\geq C\,d\ln(d/\delta)$, using $m=\max\{C_0\,\frac{d^2}{\delta},\
d+\text{e}^2\frac{k}{\delta}\}$, we obtain that 
\begin{align*}
R&\leq a\,C_2\, d\ln(d/\delta)\leq k/2,
\quad\text{w.p.}\quad
  \geq 1- \text{e}^2\,\frac{k-d}{m-d}\geq 1 -\delta,
\end{align*}
showing that $|T|\geq k-R\geq C_1\, d\ln(d/\delta)$ and $k\leq
2|T|$.

Therefore, by Lemmas \ref{l:overestimate},
\ref{l:without-replacement} and
\ref{l:geometric-tail}, there is a $1-3\delta$ probability event in which
\[
  \lambda_{\min}\Big(\frac{1}{|T|}\U^\top\Q_{\pi_T}\U\Big)
  \geq \frac14 \quad\text{and}\quad k \leq 2|T| .
\]
In this same event,
\[
  \lambda_{\min}\Big(\frac{1}{k}\U^\top\Q_{\pi_S}\U\Big)
  \geq
  \lambda_{\min}\Big(\frac{1}{k}\U^\top\Q_{\pi_T}\U\Big)
  \geq
  \lambda_{\min}\Big(\frac{1}{2|T|}\U^\top\Q_{\pi_T}\U\Big)
  \geq
  \frac12 \cdot \frac14 = \frac18 
  ,
\]
which completes the proof of Theorem \ref{t:spectral}.

\subsection{Proof of Lemma \ref{l:overestimate}}
\label{a:overestimate}
Replacing vector $\r$ in Theorem \ref{t:multiplication} with each
column of matrix $\U$, we obtain that for $m\geq C\,\frac{d}{\epsilon}$,
\begin{align*}
\E\big[\|\U^\top\Q_\pi\U - \U^\top\U\|_F^2\big]\leq \epsilon\,\|\U\|_F^2 =
  \epsilon\,d.
\end{align*}
We bound the 2-norm by the Frobenius norm and use Markov's inequality,
showing that w.p. $\geq 1-\delta$
\begin{align*}
\|\U^\top\Q_\pi\U-\I\|\leq \|\U^\top\Q_\pi\U - \I\|_F\leq 
  \sqrt{\epsilon\, d/\delta}. 
\end{align*}
Setting $\epsilon=\frac{\delta}{4d}$, for
$m\geq C_0\,d^2/\delta$, the above inequality  implies that
\begin{align*}
\lambda_{\min}\Big(\frac{1}{m}\U^\top\Q_\pi\U\Big) \geq \frac{1}{2}.
\end{align*}

\subsection{Proof of Lemma \ref{l:without-replacement}}
\label{a:without-replacement}
Let $\pi$ denote the sequence of $m$ indices selected by volume sampling in
the first stage. Suppose that $i_1,...,i_k$ are independent uniformly sampled indices
from $[m]$, and let $j_1,...,j_k$ be indices sampled uniformly
without replacement from $[m]$. We define matrices 
\begin{align*}
\Z\defeq \sum_{t=1}^k\overbrace{\frac{1}{kq_{i_t}}\u_{i_t}\u_{i_t}^\top}^{\Z_t},\quad\text{and}\quad
  \Zbh\defeq\sum_{t=1}^k\overbrace{\frac{1}{kq_{j_t}}\u_{j_t}\u_{j_t}^\top}^{\Zbh_t}. 
\end{align*}
Note that $\|\Z_t\|=\frac{d}{k\,l_i}\|\u_{i_t}\|^2=\frac{d}{k}$
and, similarly, $\|\Zbh_t\|= \frac{d}{k}$. Moreover, 
\begin{align*}
\E[\Z] =
  \sum_{t=1}^k\bigg[\frac{1}{m}\sum_{i=1}^m\frac{1}{kq_i}\u_i\u_i^\top\bigg] =  k\
  \frac{1}{k}\frac{1}{m}\U^\top\Q_\pi\U = \frac{1}{m}\U^\top\Q_\pi\U.
\end{align*}
Combining Chernoff's inequality with the reduction argument described
in \cite{uniform-matrix-sampling}, for any $\lambda$, and $\theta>0$,
\begin{align*}
\Pr\big(\lambda_{\max}(-\Zbh)\geq
  \lambda\big)\leq \text{e}^{-\theta \lambda}\
  \E\Big[\tr\big(\exp(\theta (-\Zbh))\big)\Big]
\leq
\text{e}^{-\theta \lambda}\
  \E\Big[\tr\big(\exp(\theta (-\Z))\big)\Big].
\end{align*}
Using matrix Chernoff bound of \cite{matrix-tail-bounds} applied to
$-\Z_1,...,-\Z_k$ with appropriate $\theta$, we have
\begin{align*}
\text{e}^{-\theta \lambda}\
  \E\Big[\tr\big(\exp(\theta (-\Z))\big)\Big]\leq d\
  \exp\Big(-\frac{k}{16d}\Big),\quad
\text{for}\quad \lambda = \frac{1}{2}\,\lambda_{\max}\Big(-\frac{1}{m}\U^\top\Q_\pi\U\Big)\leq-\frac{1}{4}.
\end{align*}
Thus, there is a constant $C_1$ such that for $k\geq C_1\,d\ln(d/\delta)$,
w.p.~at least $1-\delta$ we have $\lambda_{\min}(\Zbh)\geq
\frac{1}{4}$.

\subsection{Proof of Lemma \ref{l:geometric-tail}}
\label{a:geometric-tail}
We compute the moment generating function of the variable
$R_t\sim\operatorname{Geom}(r_t)$, where $r_t=\frac{t-d}{t}$:
\begin{align*}
\E\big[\text{e}^{\theta R_t}\big] =
  \frac{r_t}{1-(1-r_t)\text{e}^{\theta}}=
\frac{\frac{t-d}{t}}{1-\frac{d}{t}\,\text{e}^{\theta}} = \frac{t-d}{t-d\,\text{e}^{\theta}}.
\end{align*}
Setting $\theta=\frac{1}{2d}$, we observe that $d\text{e}^{\theta}\leq
d+1$, and so $\E[\text{e}^{\theta R_t}]\leq
\frac{t-d}{t-d-1}$. Letting $\mu=\E[R]$,
for any $a>1$ using Markov's inequality we have
\begin{align*}
\Pr(R\geq a\mu)\leq \text{e}^{-a\theta\mu}\,\E\big[\text{e}^{\theta
  R}\big]
\leq  \text{e}^{-a\theta\mu}\prod_{t=k+1}^m\frac{t-d}{t-d-1}=
\text{e}^{-a\theta\mu}\,\frac{m-d}{k-d}.
\end{align*}
Note that using the bounds on the harmonic series we can estimate the
mean:
\begin{align*}
\mu &= d\!\!\sum_{t=k-d+1}^{m-d}\frac{1}{t}\geq d\, (\ln(m-d) -
  \ln(k-d)-1)= d\,\ln\Big(\frac{m-d}{k-d}\Big) - d,\\
\text{so}\quad \text{e}^{-a\theta\mu} &
\leq \text{e}^{a/2}\,\exp\bigg(-\frac{a}{2}\ln\Big(\frac{m-d}{k-d}\Big)\bigg)=
\text{e}^{a/2}\,\Big(\frac{m-d}{k-d}\Big)^{-a/2}.
\end{align*}
Putting the two inequalities together we obtain the desired tail bound.

\section{Experiments}
\label{sec:experiments}

We present experiments comparing leveraged volume sampling to standard
volume sampling and to leverage score sampling, in terms of the
total square loss suffered by the subsampled least-squares
estimator. The three estimators can be summarized as follows:
\begin{align*}
\textit{volume sampling:} \quad\w_S^* &= (\X_S)^+\y_S,&\Pr(S)&\sim
  \det(\X_S^\top\X_S),
\quad S\in {[n]\choose k};\\
\textit{leverage score sampling:}\quad
 \w_\pi^* &=(\Q_\pi^{\sfrac12}\X)^+\Q_\pi^{\sfrac12}\y,
&\Pr(\pi) &= \prod_{i=1}^k\frac{l_{\pi_i}}{d},\qquad\qquad\pi\in[n]^k;\\
\textit{leveraged volume sampling:}\quad
 \w_\pi^* &=(\Q_\pi^{\sfrac12}\X)^+\Q_\pi^{\sfrac12}\y,
&\Pr(\pi) &\sim \det(\X^\top\Q_\pi\X)\prod_{i=1}^k\frac{l_{\pi_i}}{d} .
\end{align*}
Both the volume sampling-based estimators are unbiased, however
the leverage score sampling estimator is not. Recall that
$\Q_\pi=\sum_{i=1}^{|\pi|}q_{\pi_i}^{-1}\e_{\pi_i}\e_{\pi_i}^\top$ is the selection
and rescaling matrix as  defined for
$q$-rescaled volume sampling with $q_i=\frac{l_i}{d}$. For each
estimator we plotted its average total loss,
i.e., $\frac{1}{n}\|\X\w-\y\|^2$, for a range of sample sizes $k$,
contrasted with the loss of the  best least-squares estimator $\w^*$
computed from all data. 

\begin{wrapfigure}{l}{0.45\textwidth}
\begin{tabular}{c|c|c}
Dataset & Instances ($n$) & Features ($d$) \\
\hline
\textit{bodyfat} & 252& 14\\
\textit{housing} &506& 13\\
\textit{mg} & 1,385 & 21\\
\textit{abalone} & 4,177  & 36 \\
\textit{cpusmall} & 8,192 &12\\ 
\textit{cadata} & 20,640&8\\
\textit{MSD} &463,715&90
\end{tabular}
\captionof{table}{Libsvm regression datasets \cite{libsvm} (to
  increase dimensionality of \textit{mg} and \textit{abalone}, we
  expanded features to all degree 2 monomials, and removed redundant
  ones).}   
\label{tab:datasets}
\end{wrapfigure}

Plots shown in Figures \ref{f:lb} and \ref{fig:experiments} were
averaged over 100 runs, with shaded area representing standard error
of the mean. We used seven benchmark datasets from the libsvm
repository \cite{libsvm} (six in this section and one in Section
\ref{s:intro}), whose dimensions are given in Table
\ref{tab:datasets}. The results confirm that leveraged volume sampling
is as good or better than either of the baselines for any sample size
$k$. We can see that in some of the examples standard volume sampling
exhibits bad behavior for larger sample sizes, as suggested by the
lower bound of Theorem \ref{t:lower} (especially noticeable on
\textit{bodyfat} and \textit{cpusmall} datasets). On the other hand, leverage
score sampling exhibits poor performance for small sample sizes due to
the coupon collector problem, which is most noticeable for
\textit{abalone} dataset, where we can see a very sharp transition
after which leverage score sampling becomes effective. Neither of the
variants of volume sampling suffers from this issue.

\begin{figure}
\includegraphics[width=0.5\textwidth]{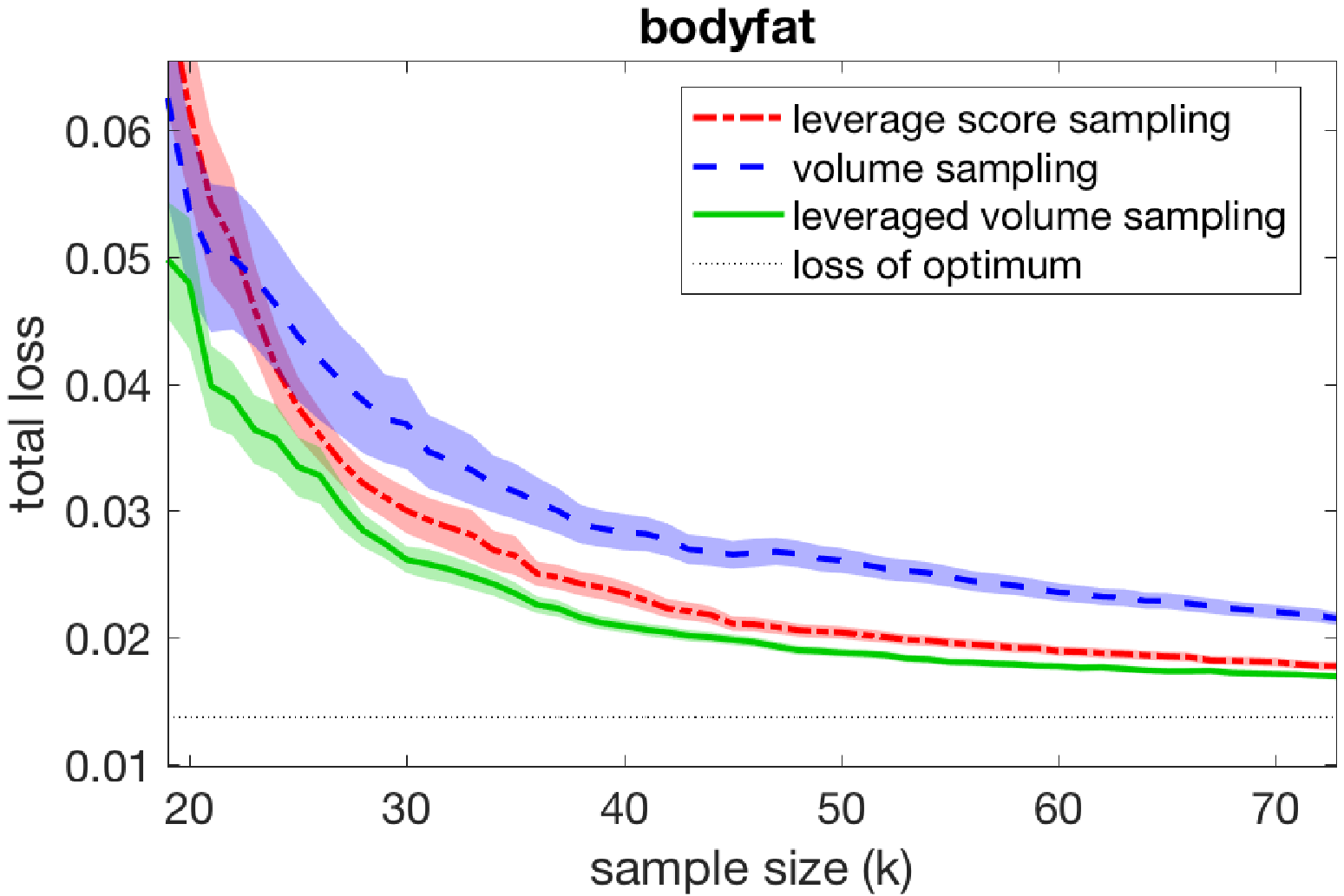}\nobreak
\includegraphics[width=0.5\textwidth]{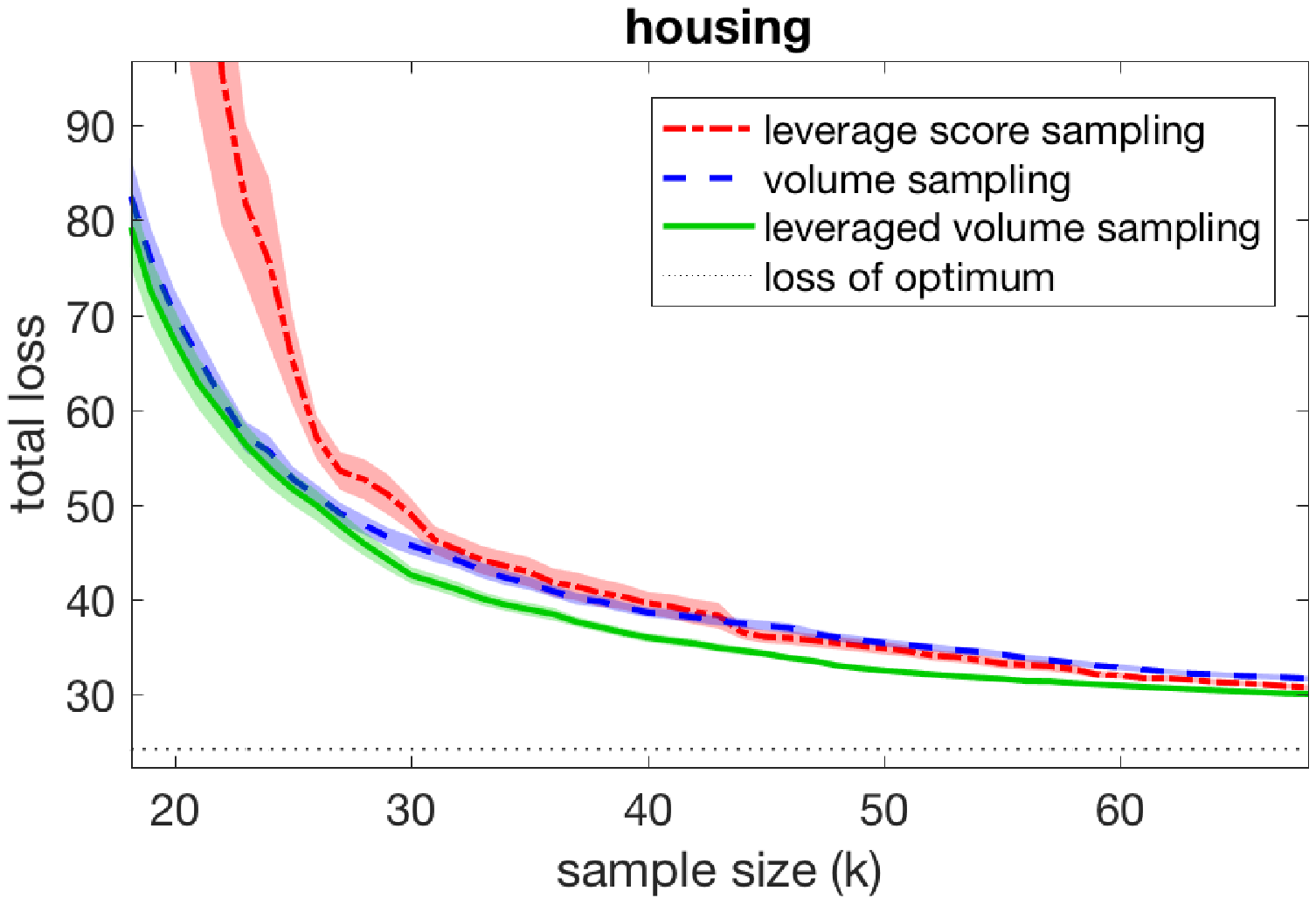}
\includegraphics[width=0.5\textwidth]{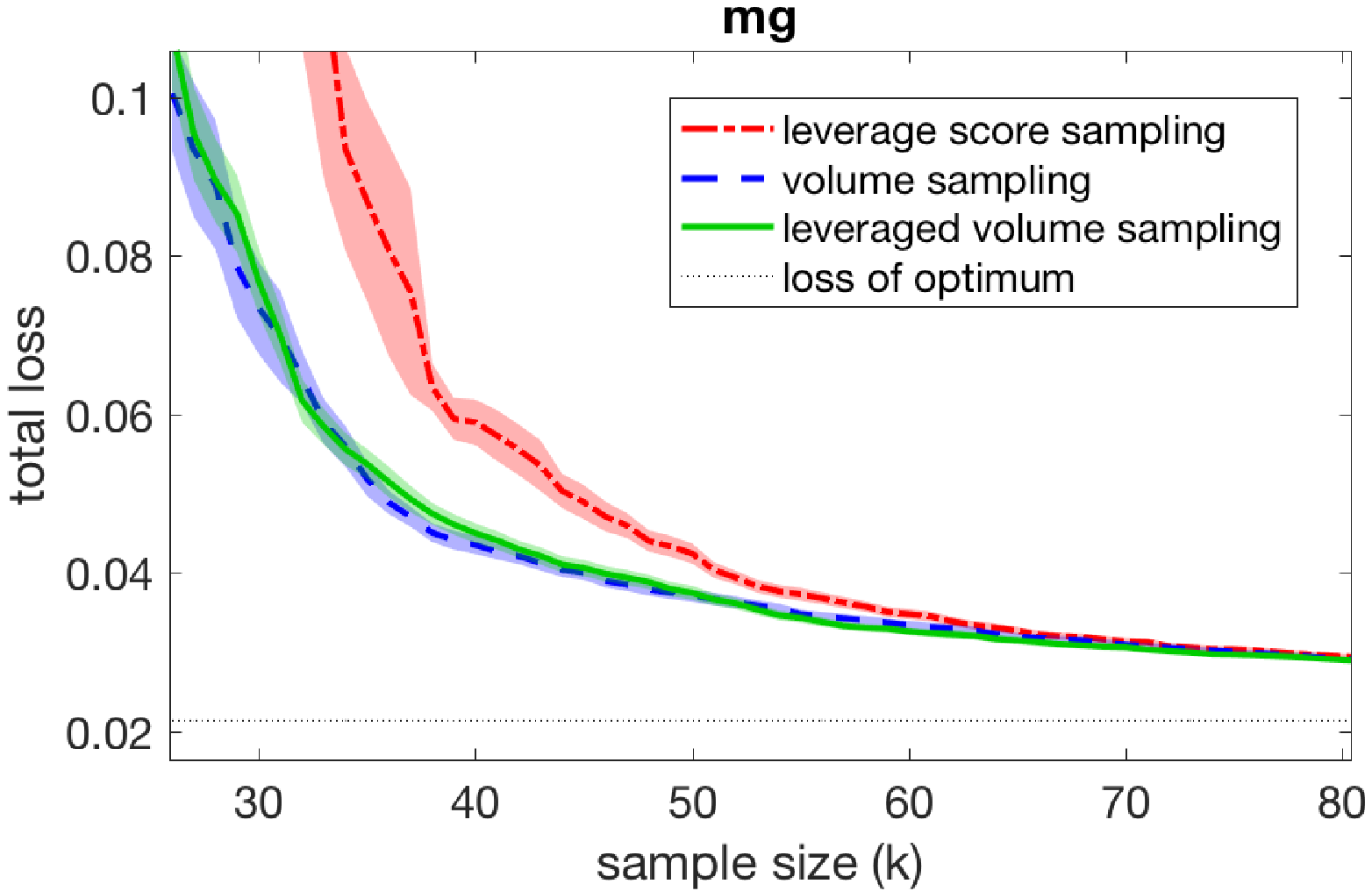}\nobreak
\includegraphics[width=0.5\textwidth]{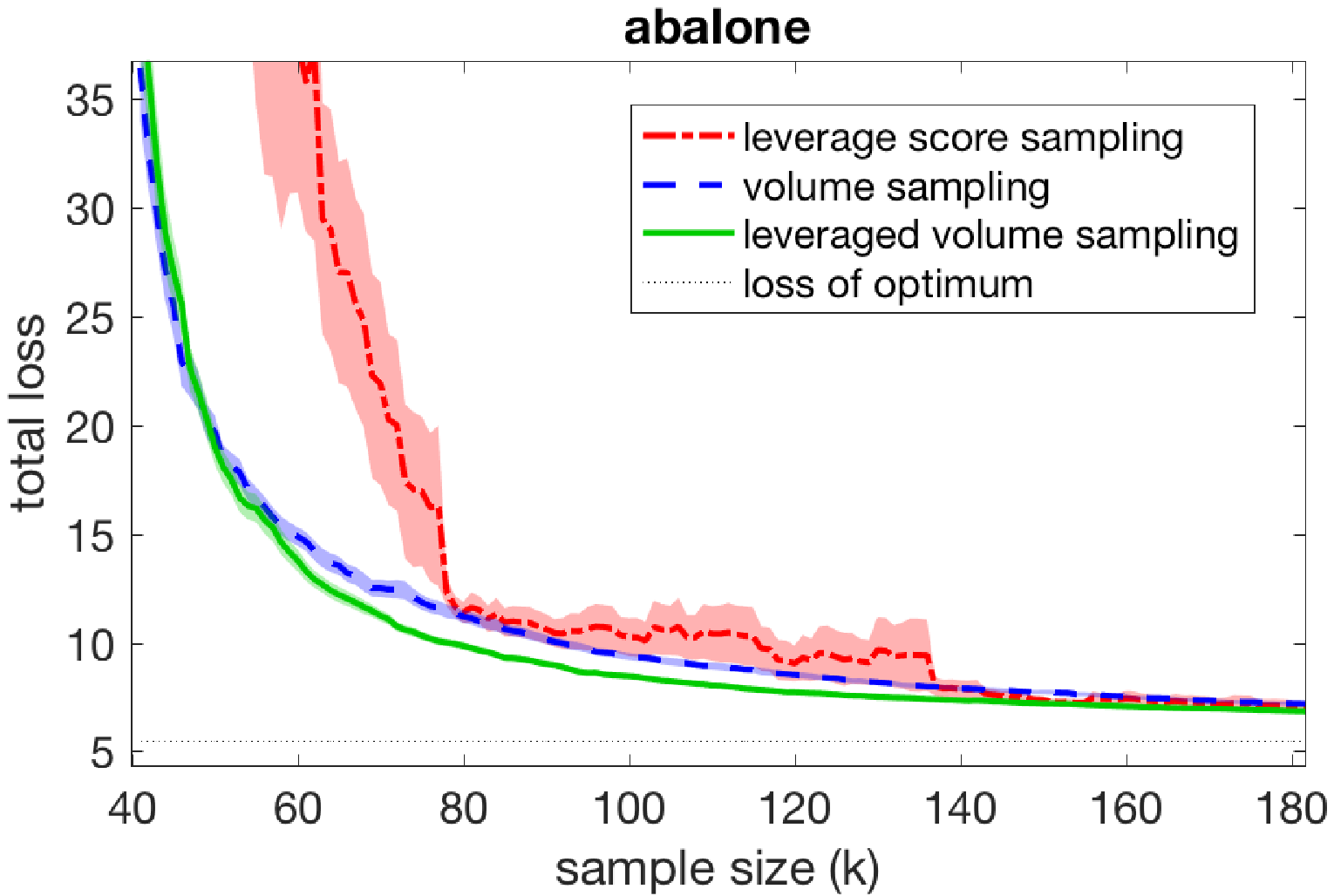}
\includegraphics[width=0.5\textwidth]{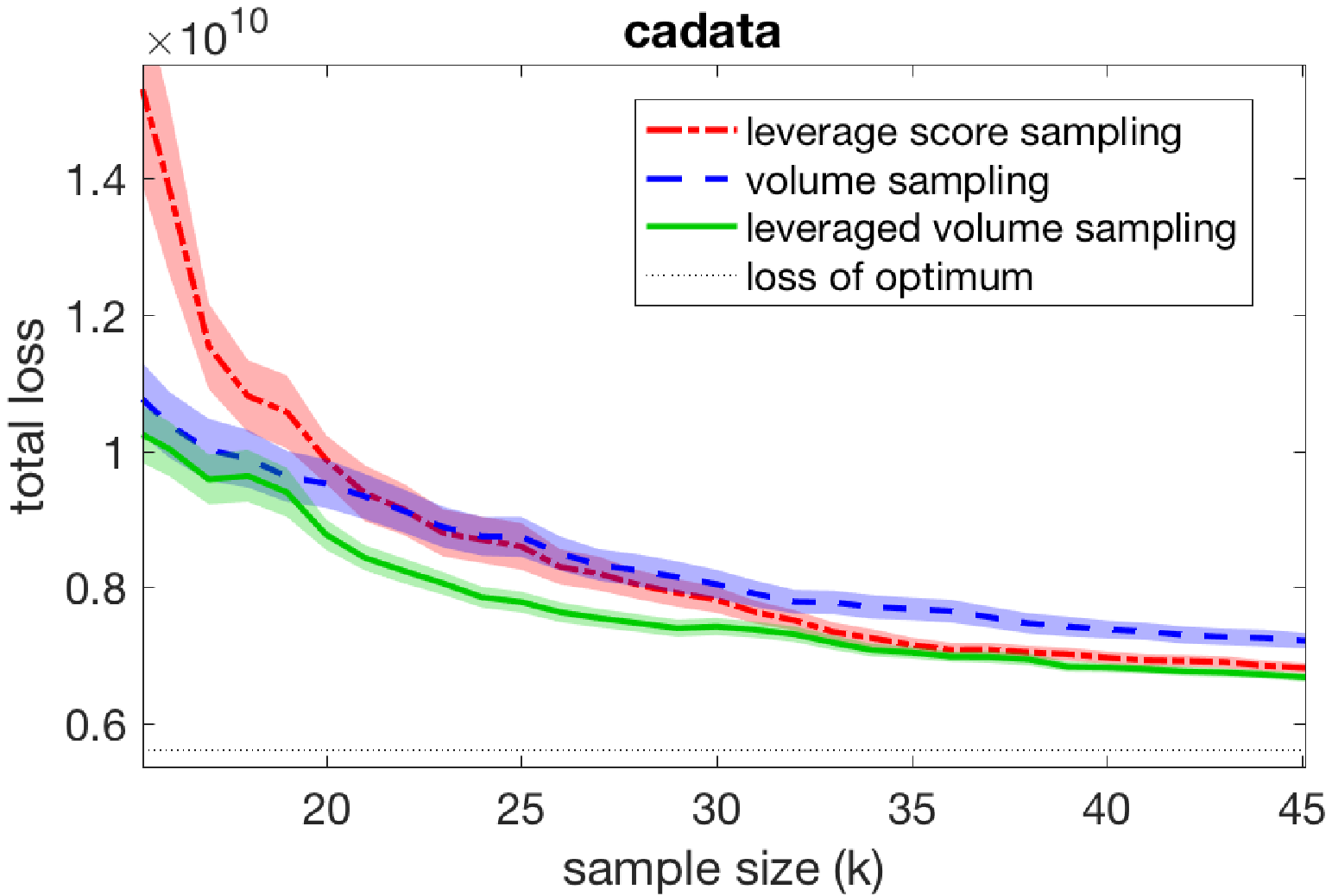}\nobreak
\includegraphics[width=0.5\textwidth]{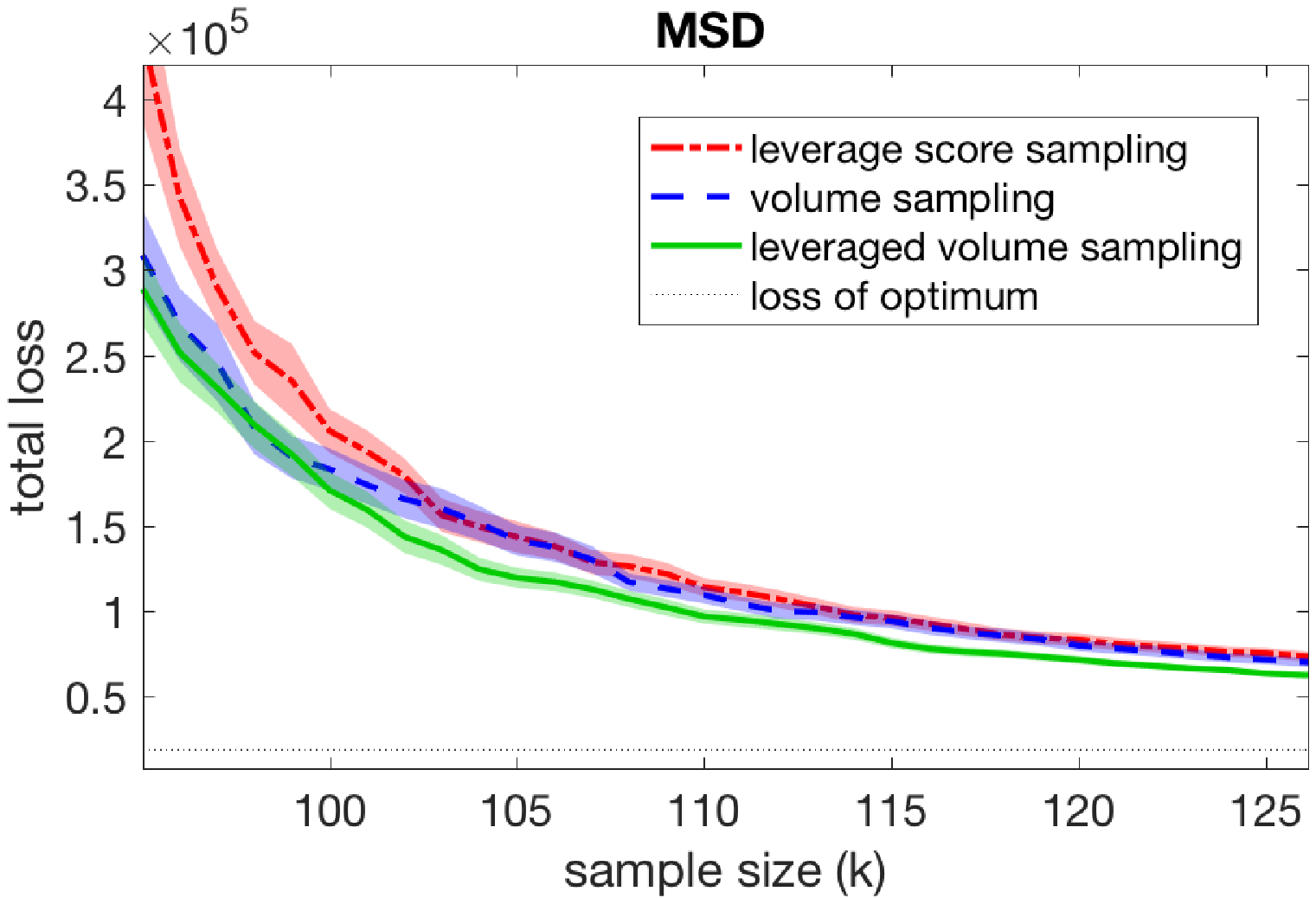}
\caption{Comparison of loss of the subsampled estimator when
  using \textit{leveraged volume sampling} vs using \textit{leverage score sampling} and
  standard \textit{volume sampling}  on six datasets.}
\label{fig:experiments}
\end{figure}

\section{Faster algorithm via approximate leverage scores}
\label{sec:fast-alg}

\begin{wrapfigure}{r}{0.4\textwidth}
\renewcommand{\thealgorithm}{}
\vspace{-9mm}
\begin{minipage}{0.4\textwidth}
\floatname{algorithm}{}
\begin{algorithm}[H] 
{\fontsize{8}{8}\selectfont
  \caption{\bf \small Fast leveraged volume sampling}
  \begin{algorithmic}[0]
    \STATE \textbf{Input:} $\X\!\in\!\R^{n\times d},\, k\geq
    d,\,\epsilon\geq 0$\\[1mm]
    \STATE Compute $\A = (1\pm \epsilon) \,\X^\top\X$
    \STATE Compute $\tilde{l}_i=(1\pm \frac12)\, l_i\quad \forall_{i\in[n]}$
    \STATE $s \leftarrow \max\{k,\,8d^2\}$
    \STATE \textbf{repeat}
    \STATE \quad $\pi \leftarrow$ empty sequence
    \STATE \quad\textbf{while} $|\pi|<s$
    \STATE \quad\quad Sample $i\ \sim\
    (\tilde{l}_1,\dots,\tilde{l}_n)$
    \STATE \quad\quad $a\sim
    \text{Bernoulli}\Big((1\!-\!\epsilon)\frac{\x_i^\top\A^{-1}\x_i}{2\tilde{l}_i}\Big)$
    \STATE \quad\quad\textbf{if} $a=\text{true}$,\quad\textbf{then}\quad $\pi \leftarrow
    [\pi, i]$
    \STATE \quad\textbf{end}\\[1mm]
    \STATE \quad $\Q_\pi\leftarrow \sum_{j=1}^sd\,
    (\x_{\pi_j}^\top\A^{-1}\x_{\pi_j})^{-1}\e_{\pi_j}\e_{\pi_j}^\top$
\vspace{-1mm}
    \STATE \quad Sample $\textit{Acc}\sim
    \text{Bernoulli}\Big(\frac{\det(\frac{1}{s}\X^\top\Q_\pi\X)}{\det(\A)}\Big)$
\vspace{-1mm}
    \STATE \textbf{until} $\textit{Acc}=\text{true}$
    \STATE $S\leftarrow$ VolumeSample$\big((\Q_{[1..n]}^{\sfrac{1}{2}}\X)_\pi,k\big)$
    \RETURN $\pi_S$
 \end{algorithmic}
}
\end{algorithm}
\end{minipage}
\vspace{-5mm}
\end{wrapfigure}

In some settings, the primary computational cost of deploying leveraged volume
sampling is the preprocessing cost of computing exact laverage
scores for matrix $\X\in\R^{n\times d}$, which takes $O(nd^2)$. There
is a large body of work dedicated to fast estimation of leverage
scores (see, e.g., \cite{fast-leverage-scores,randomized-matrix-algorithms}),
and in this section we examine how these approaches can be
utilized to make leveraged volume sampling more efficient. The key
challenge here is to show that the determinantal rejection sampling
step remains effective when distribution $q$ consists of approximate
leverage scores. Our strategy, which is described in the algorithm
\textit{fast leveraged volume sampling}, will be to compute an
approximate covariance matrix $\A=(1\pm\epsilon)\X^\top\X$ and use it
to compute the rescaling distribution $q_i\sim
\x_i^\top\A^{-1}\x_i$. As we see in the lemma below, for sufficiently small
$\epsilon$, this rescaling still retains the runtime guarantee of
determinantal rejection sampling from Theorem \ref{t:algorithm}.\\

\begin{lemma}\label{l:fast-rejection}
Let $\X\in\R^{n\times d}$ be a full rank matrix, and suppose that matrix $\A\in\R^{d\times d}$ satisfies
\begin{align*}
(1-\epsilon)\,\X^\top\X\preceq \A\preceq
  (1+\epsilon)\,\X^\top\X,\quad \text{where}\quad \frac{\epsilon}{1-\epsilon}\leq\frac{1}{16d}.
\end{align*}
Let $\pi_1,\dots,\pi_s$ be sampled i.i.d. 
$\sim(\hat{l}_1,\dots,\hat{l}_n)$, where
$\hat{l}_i=\x_i^\top\A^{-1}\x_i$. If $s\geq 8d^2$, then
\begin{align*}
\text{for}\quad\Q_\pi=\sum_{j=1}^s\frac{d}{\hat{l}_{\pi_j}}\e_{\pi_j}\e_{\pi_j}^\top,\qquad
\frac{\det(\frac{1}{s}\X^\top\Q_\pi\X)}{\det(\A)}\leq 1\quad \text{and}\quad\E\bigg[\frac{\det(\frac{1}{s}\X^\top\Q_\pi\X)}{\det(\A)}\bigg]\geq \frac{3}{4}.
\end{align*}
\end{lemma}
Proof of Lemma \ref{l:fast-rejection} follows along the same lines as
the proof of Theorem \ref{t:algorithm}. We can compute matrix $\A^{-1}$
efficiently in time $\widetilde{O}(nd + d^3/\epsilon^2)$ using a sketching
technique called Fast Johnson-Lindenstraus Transform~\cite{ailon2009fast}, as described in
\cite{fast-leverage-scores}. However, the cost of computing the entire
rescaling distribution is still $O(nd^2)$. Standard techniques
circumvent this issue by performing a second matrix sketch. We cannot
afford to do that while at the same time preserving the sufficient
quality of leverage score estimates needed for leveraged volume
sampling. Instead, we first compute weak estimates
$\tilde{l}_i=(1\pm\frac{1}{2})l_i$ in time $\widetilde{O}(nd+d^3)$ as
in \cite{fast-leverage-scores}, then use rejection sampling to sample
from the more accurate leverage score distribution, and finally compute the
correct rescaling coefficients just for the obtained sample. Note that
having produced matrix $\A^{-1}$, computing a single leverage score
estimate $\hat{l}_i$ takes $O(d^2)$. The proposed algorithm with high
probability only has to compute $O(s)$ such estimates, which
introduces an additional cost of $O(sd^2) = O((k+d^2)\,d^2)$. Thus, as long as
$k=O(d^3)$, dominant cost of the overall procedure still comes from
the estimation of matrix $\A$, which 
takes $\widetilde{O}(nd+d^5)$ when $\epsilon$ is chosen as in Lemma
\ref{l:fast-rejection}.

It is worth noting that \textit{fast leveraged volume sampling} is
a valid $q$-rescaled volume sampling distribution (and not an
approximation of one), so the least-squares estimators it produces
are exactly unbiased. Moreover, proofs of Theorems \ref{t:multiplication} and
\ref{t:spectral} can be straightforwardly extended to the setting
where $q$ is constructed from approximate leverage scores, so our loss
bounds also hold in this case.

\end{document}